\documentclass[journal]{IEEEtran}
\usepackage{graphicx}
\usepackage{epstopdf}
\usepackage{subfigure}
\usepackage{amsthm}
\usepackage{amsmath}
\usepackage{color}
\usepackage{algorithm}
\usepackage{algorithmicx}
\usepackage[noend]{algpseudocode}
\usepackage[normalem]{ulem}
\usepackage{amssymb}
\usepackage{mathrsfs}
\usepackage{diagbox}
\usepackage{booktabs}
\usepackage{multirow}
\newtheorem{theorem}{Theorem}
\newtheorem{lemma}{Lemma}
\newtheorem{corollary}{Corollary}
\begin{document}

\title{Uncertainty Minimization for Personalized Federated Semi-Supervised Learning  \thanks{Y. Shi and S. Chen are with the School of Internet of Things, Nanjing University of Posts and Telecommunications, Nanjing, China; H. Zhang is with the School of Computer and Communication Engineering, University of Science and Technology Beijing, Beijing, China (e-mail: yanhang98@126.com; sgchen@njupt.edu.cn; haijunzhang@ieee.org). \emph{Corresponding author: Siguang Chen}.}}

\author{Yanhang~Shi, Siguang~Chen, \emph{Member}, \emph{IEEE}, and~Haijun~Zhang, \emph{Senior Member}, \emph{IEEE}}
\maketitle

\begin{abstract}
Since federated learning (FL) has been introduced as a decentralized learning technique with privacy preservation, statistical heterogeneity of distributed data stays the main obstacle to achieve robust performance and stable convergence in FL applications. Model personalization methods have been studied to overcome this problem. However, existing approaches are mainly under the prerequisite of fully labeled data, which is unrealistic in practice due to the requirement of expertise. The primary issue caused by partial-labeled condition is that, clients with deficient labeled data can suffer from unfair performance gain because they lack adequate insights of local distribution to customize the global model. To tackle this problem, 1) we propose a novel personalized semi-supervised learning paradigm which allows partial-labeled or unlabeled clients to seek labeling assistance from data-related clients (helper agents), thus to enhance their perception of local data; 2) based on this paradigm, we design an uncertainty-based data-relation metric to ensure that selected helpers can provide trustworthy pseudo labels instead of misleading the local training; 3) to mitigate the network overload introduced by helper searching, we further develop a helper selection protocol to achieve efficient communication with acceptable performance sacrifice. Experiments show that our proposed method can obtain superior performance and more stable convergence than other related works with partially labeled data, especially in highly heterogeneous setting.
\end{abstract}


\maketitle

\IEEEpeerreviewmaketitle

\section{Introduction}\label{sec:introduction}
\IEEEPARstart{O}{ver} the recent years, with the expansion of personal mobile devices and wireless networks, the amount of user data has witnessed tremendous growth at the edge side. To reap the benefits of large-scale data, machine learning based smart services have been investigated in massive studies (e.g., image classification [1], nature language processing [2] and speech recognition [3]), and achieved remarkable success. However, these techniques require to gather scattered user data into a single dataset and train model centrally. While in real-world applications, collecting distributed data to a central server is usually unrealistic due to the data privacy and commercial competition. In this regard, federated learning (FL) [4] has been proposed as a distributed learning algorithm which aims to collaboratively train a global deep learning model without sharing local data. Unlike centralized training, engaged clients in FL undertake the model training task, while the central server conducts the information aggregation via the weighted average of local parameters. In canonical FL, a single training round mainly includes following steps: 1) the central server samples a subset of clients and distributes the global model to them; 2) selected clients train the global model with their private data and upload the modified models back to server; 3) the central server aggregates the local models to update the global model, which is used for next training round. In FL system, the server and clients iteratively perform the above three stages until the model converged.

\par With the communication-efficient pattern and privacy-preserved information aggregation procedure, FL has been widely utilized in advanced applications, such as Internet of Things (IoT) [5] and smart healthcare [6]. Nevertheless, FL still faces three major challenges to obtain guaranteed convergence rate and appropriate generalization ability to unseen data. One challenge is the statistic heterogeneity of users' local datasets, also named non-independent and identically distributed (Non-IID) data. In this setting, the local models learned by clients are prone to interfere each other in the aggregation stage, leading to unstable convergence rate and inferior performance. Several prior works have been proposed to address the issues caused by Non-IID data. For instance, Li et al. in [7] proposed the regularization method to constrain the model drift in the local training stage, which achieves more stable convergence than conventional FL. Other studies like [8] sought to adjust the gradient directions of different clients to prevent the interference involved in central aggregation.
\par While these works effectively resolve the model conflicting effect induced in Non-IID FL, modifying or eliminating the skew of local parameters might cause the performance unfairness, which means the generalization ability of the global model to different clients diverges to a large extent. Building a model that can be fairly generalized to distinct clients is the second challenge of FL. In order to preserve the local update directions and mitigate the performance bias, some studies [9], [10] and [11] considered the model personalization in federated learning, namely the personalized federated learning (pFL). The pFL based methods allow each participant to maintain a personalized model while still derive benefit from collaborative training. Under the strong assumption that all on-device data are fully labeled, these methods have achieved the comparably fair performance in each client.
\par In realistic applications, the available labeled data distributed in multiple sources can be limited. For instance, in medical image classification field, due to the laborious and expert-requiring labeling procedure, the labeled data can be quite restricted in personal devices or non-professional institutions. To overcome this limitation, some previous works proposed to directly implement conventional semi-supervised methods in client side [12] or modify these techniques to adjust the FL setting [13]. However, these works ignore the primary issue of federated learning in semi-supervised scenarios, which is the performance fairness for clients. Under the partial-labeled and Non-IID data condition, clients with scarce labeled data will struggle to obtain competent personalized models due to insufficient knowledge of their local data distributions, while those with adequate labeled data can dominate the collaborative training and obtain superior performance. Thus, concern about performance fairness resurfaces as a main challenge for FL applications.
\par In this paper, we study the uncertainty minimization for personalized federated semi-supervised learning (UM-pFSSL) to address these three challenegs in semi-supervised scenario: (1) data heterogeneity of engaged clients; (2) the lack of knowledge to personalized the global model; (3) performance fairness to local data distributions. The main contributions of which include:
\par 1) We present a novel semi-supervised learning paradigm for pFL, which aims to tackle the performance bias that stems from partial-labeled and Non-IID data in realistic applications. In this paradigm, partial-labeled or unlabeled clients can selectively extract knowledge from data-related clients to annotate the unlabeled data locally. Supervised training on these pseudo labels enables the clients efficiently mining the information of unlabeled data and resolves the dilemma of insufficient insights into local data distributions, so as to achieve fairer performance. After that, benefited clients can further spread out their knowledge extracted from federation by providing labeling assistance to others, which strengthens the knowledge sharing efficiency of FL system.
\par 2) A potential threat of learning from peer nodes is that, clients who received detrimental knowledge (i.e., counterfactual pseudo labels) tend to promptly poison the entire federation through the knowledge feedback procedure. To obviate this risk, we define an uncertainty-based data-relation metric. Based on which, each client can effectively select relevant helpers which are more apt to provide trustworthy predictions for local unlabeled data. Furthermore, with the theoretical guarantee, aggregating the parameters from selected helpers can further improve the generalization ability of each client.
\par 3) Greedily searching helpers from federation is prone to bring huge communication cost to the system. To exempt the extra network burden, we design a ranking update protocol for helper selection. Specifically, rather than applying selection to the complete client group, each client chooses to randomly download a small subset of external models and evaluates the data-relation accordingly. Thereafter, irrelevant models of former rounds will be replaced with more relevant ones in current round. This replacing procedure only executes in the early few rounds, and selected helpers will be updated periodically through the server's model pool. By applying this protocol, our scheme saves considerable communication overhead with acceptable performance sacrifice.
\par Finally, we compare the UM-pFSSL with Non-IID methods and personalized methods in heterogeneous partial-labeled setting. The results show that our proposed scheme can achieve more robust performance, superior convergence rate and comparable fairness than compared methods without introducing extra communication cost.
\par The rest of this paper is arranged as follows. Section \ref{sec:related} covers the related works. In Section \ref{sec:paradigm}, the designed semi-supervised learning paradigm is described. In Section \ref{sec:method}, the UM-pFSSL algorithm is given in detail. The theoretical analysis of the generalization ability is elaborated in Section \ref{sec:theory}. Then, the experiment results are presented in Section \ref{sec:result}. Finally, conclusions are drawn in Section \ref{sec:conclusion}.

\section{Related work} \label{sec:related}
\subsection{Federated Learning with Non-IID Data}
\par Due to the advantages in privacy preservation and communication efficiency, federated learning has become an appealing research field. The idea of the standard federated learning method FedAvg [4] is to directly conduct the weight average on local parameters of clients, which fails to obtain robust performance and convergence rate with Non-IID data. A number of prior works have attempted to mitigate the deficiency of FedAvg. In [7], the authors leveraged a regularization term to close the divergence between local and global parameters. By applying which, the global information from previous rounds can avoid being forgotten in the local training stage. Instead of using the global parameter to bound the local drift, works [14], [15] and [8] investigated to modify the local gradient by computing a control variable for each client. For example, authors in [15] utilized a dynamic regularizer for each device at each round, which can align the local-level update to the global direction. While this type of method successfully enhances the convergence rate in label distribution shift FL and achieves superior performance than conventional method FedAvg, some studies like [16] and [17] considered to deal with the feature shift in Non-IID data. For instance, based on the utility of batch normalization (BN), Li et al. in [17] developed a novel method that excludes the BN layer in the averaging step, which can theoretically improve the convergence rate and performance robustness with variant client features. Although these works developed efficacious strategies to overcome the issues in Non-IID FL, due to the universal average operation, the target distribution of global model can be divergent to local data distributions. Thus, the resulting global model does not confirm the unbiased performance on different clients.
\subsection{Personalized Federated Learning}
\par Different from above single-model based schemes, other works aim to build personalized models for individual clients while still draw shared knowledge from federation. The idea of pFL was first proposed in [18] with a multi-task learning setting, which applies a task-specific model to each client. Inspired by this design, some works [19], [20] and [21] conducted the clustering approaches to group clients to smaller cliques based on the similarity of their local distributions and learn specialized model for each clique. For instance, Briggs et al. in [20] proposed a hierarchical clustering approach that separates the clients by the similarity of their local modifications to the global model and creates independent model for each cluster. By sharing representations between related clients, the clustering schemes can improve generalization of each model. Mixed model is also a prevalent learning scheme in pFL, several literatures such as [9], [22] and [23] considered to decompose neural network to the global and local parts. As a typical example, Arivazhagan et al. in [9] tried to only aggregate the bottom layers near the input side through federated average and remain the top layer with local updated parameters. The purpose of these studies is to prevent overfitting issue by aggregating the feature extraction capabilities learned by different clients while preserving the knowledge of the local data distributions. Besides above methods, some works [24] and [25] studied to utilize meta-learning-like methods to customize the global trained model to specific client. In [25], Fallah et al. utilized the model-agnostic meta-learning (MAML) framework to find an initial global model, which can be easily adapt to clients' personal datasets by performing a few steps of local training. These methods provide relatively fair concern to every single user and achieved notable results; however, they lack the consideration of different labeling capabilities of individuals, which requires semi-supervised learning techniques to improve the performance of limited labeled datasets.
\subsection{Semi-Supervised Learning}
\par As an attractive topic in centralized machine learning, semi-supervised learning (SSL) seeks to tackle the label scarcity by leveraging the unlabeled data. As a powerful approach for SSL, consistency regularization gives the intuition that, a robust neural network should generate invariant predictions for the same unlabeled data with semantic-preserving perturbations [26]. Based on this assumption, many related works [27], [28] and [29] proposed to combine this restriction with data augmentation methods. Another common strategy for SSL is the entropy minimization [30], which aims to reduce the prediction entropy of unlabeled data. More recently, pseudo-labeling method [31] proposed to integrate the idea of consistency regularization to entropy minimization. In detail, the authors used two different data augmentation methods to transform the unlabeled data, and supervised the strongly augmented samples with sharpened low-entropy prediction on the weakly augmented samples. Since the semi-supervised learning problem also exists in the FL scenario, some current studies [13], [32] and [33] investigated to construct SSL framework for FL based applications. As an instance, Jeong et al. in [13] presented an inter-client consistency loss and a disjoint learning pattern on labeled and unlabeled data. These works mainly concentrate on the situation that all clients are partially labeled with consistent ratio or the clients are simply categorized to labeled clients or unlabeled clients. However, in general case, clients with different identities may have diverse labeled ratios and heterogeneous data distributions, leading to the requirement of fair performance regard.

\begin{figure*}[htb]
   \centering
   \includegraphics[width=5.5in]{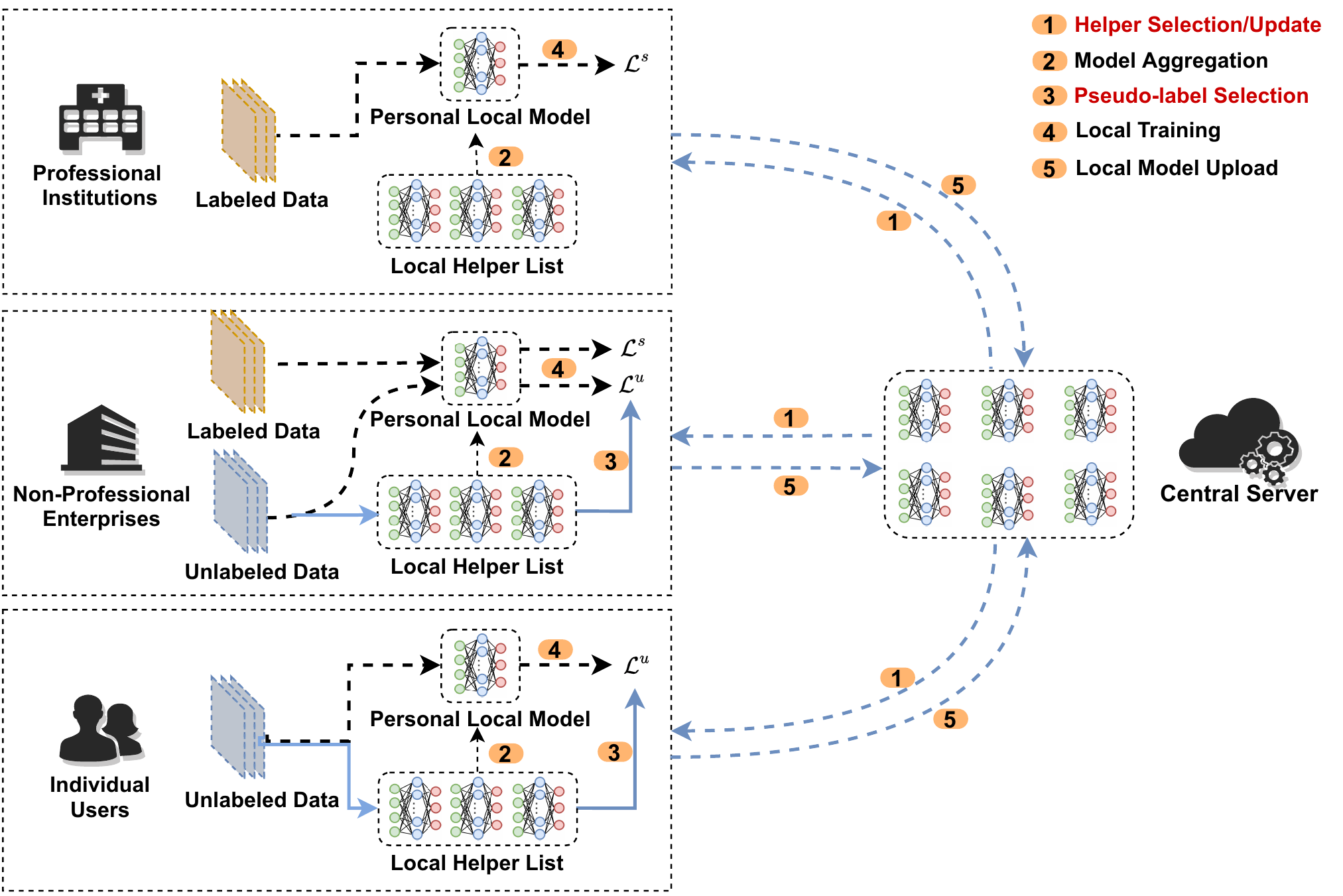}
   \caption{The semi-supervised paradigm for personalized federated learning.}
   \label{Fig.1}
\end{figure*}

\section{Learning Paradigm} \label{sec:paradigm}
\par In FL, the inter-client data heterogeneity can be mainly divided into two aspects: label distribution variance and image domain variance (which are also called label distribution skew and feature distribution skew). The former is caused by the difference of label distributions on different devices, while the latter represents the feature disparity caused by the different scanners/sensors, geographical locales and personal photography manners, which are potentially heterogeneous. Besides data heterogeneity, the annotation capabilities of participants tend to be heterogeneous in practical scenarios. As in medical applications, professional institutions or clinics possess the expertise to annotate the medical records (such as medical images), leading to mostly labeled or fully labeled datasets, while individual users lack the knowledge or time to label their data, their local datasets can be completely unlabeled. In addition, we further consider the non-professional enterprises which contain limited ratios of labeled data, due to the reason they can classify few easy samples. To be more general, the labeled-to-unlabeled ratio per clients can be arbitrarily different, which is called label heterogeneity.

\par Under the consideration of both data heterogeneity and label heterogeneity, we design a semi-supervised learning paradigm named \textit{Personalized Federated Semi-Supervised Learning} (pFSSL). Let $[K]:=\{1,... ,K\}$ be the set of all clients, which can be classified into professional institutions, non-professional enterprises and individual users according to their labeled ratios. In our paradigm, each client $k \in [K]$ carries the samples both from the private data distribution $P_{X,Y}^k$ and its marginal distribution $P_X^k$, where $X$ and $Y$ denotes the set of features and the set of labels respectively. In semi-supervised learning, these samples can be further split to local labeled dataset $D_k^s$ and unlabeled dataset $D_k^u$. The local optimization objective for semi-supervised learning is to minimize the loss function:
\begin{equation}
\mathcal{L}^{s s l}\left(w_{k}\right)=\mu_{k} \mathcal{L}^{s}\left(w_{k}\right)+\left(1-\mu_{k}\right) \mathcal{L}^{u}\left(w_{k}\right),
\end{equation}
where $\mathcal{L}^s$  and $\mathcal{L}^u$ represent the supervised loss and unsupervised loss, respectively, and $w_k$ is the client-specific model parameter. The labeled-to-all data ratio is defined as
\begin{equation}
\mu_{k}=\frac{\left|D_{k}^{s}\right|}{\left|D_{k}^{s}\right|+\left|D_{k}^{u}\right|},
\end{equation}
which is utilized to control the supervised and unsupervised ratio of (1). The goal of pFSSL is to enable client collaboration and optimize the joint loss function:
\begin{equation}
W^{*}=\underset{W}{\operatorname{argmin}} \frac{1}{K} \sum_{k=1}^{K} \mathcal{L}^{ssl}\left(w_{k}\right),
\end{equation}
where $W$ denotes the set of personalized parameters $\{w_{k}\}_{k=1}^{K}$. In this work, we simply utilize the cross entropy between the prediction and ground truth as the supervised loss function:
\begin{equation}
\mathcal{L}^{s}\left(w_{k}\right)=\sum_{\left(x_{i}, y_{i}\right) \in D_{k}^{s}} C E\left(f_{w_{k}}\left(x_{i}\right), y_{i}\right),
\end{equation}
where $f_{w_k}$ represents the neural network parameterized by $w_k$.

\par Personalized federated learning aims to handle the data heterogeneities by learning personalized models which are well-generalized to their private data distributions. However, in semi-supervised scenario, different ratios of labeled data can bias the perception of clients about their local distributions, leading to performance fairness of personalized models. In this work, we integrate a novel labeling mechanism to mitigate this problem; the five key operations of our paradigm are depicted in Fig. \ref{Fig.1} and given in detail as follows.

\par \textbf{Helper Selection/Update:} We assign the central server as the storage center of local models and all clients duplicate their models to the model pool on the server. For a single training round, each client $k \in [K]$ searches new helper agents from server according to their similarity in data or updates models of former selected helper agents. The downloaded helpers' models will be stored in the helper list $H_k$ (the detail of helper selection is described in Section \ref{sec:method}).

\par \textbf{Model Aggregation:} For knowledge aggregation, we implement the weighted average of helper' models to create updated local models (with each client $k$ itself as a helper in the list $H_k$). This procedure is similar to cluster-based approaches [19], [20] and [21], the difference is that our method is implemented at edge-side and private to each client.
\par \textbf{Pseudo-label Selection:} Then, to further personalize the aggregated model based on the local data distribution, we operate the labeling mechanism for local unlabeled data. Given an unlabeled point $x_i \in D_k^u$, the labeling mechanism tends to request the labeling assistance from helpers, and gives the pseudo prediction $\widehat{y}_i$ with minimum uncertainty (the detail of pseudo-labeling is presented in Section \ref{sec:method}).

\par \textbf{Local Training:} With confident pseudo labels, each client will operate the local training procedure by minimizing (4) and unsupervised loss function:
\begin{equation}
\mathcal{L}^{u}\left(w_{k}\right)=\sum_{x_{i} \in D_{k}^{u}} K L\left(f_{w_{k}}\left(x_{i}\right), \widehat{y}_i \right),
\end{equation}
The minimization of (5) is also named knowledge distillation procedure in many FL literatures [34], [35] and [36], which assembles the heterogeneous models by approximating their predictions to same inputs.

\par \textbf{Local Model Update:} After the local training, these clients will upload their models for updating the copies on the server.

\section{Uncertainty minimization algorithm}\label{sec:method}
\par This section first gives the definition of predictive uncertainty in deep learning and subsequently presents the Bayesian Neural Networks (BNNs) based labeling mechanism for the unlabeled data in pFSSL. Based on this foundation, we design the uncertainty minimization algorithm for pFSSL (UM-pFSSL).

\subsection{Labeling Mechanism}
\par In machine learning, the standard training procedure of supervised learning aims to maximize the log likelihood:
\begin{equation}
\max _{w} \sum_{(x, y) \sim \mathcal{D}} \log (p(y|x;w)).
\end{equation}
However, fixed parameter fails to estimate the uncertainty of predictive results on test data. BNNs tackle this problem by replacing the parameter with the posteriori $p(w|D)$, which can be obtained by Bayes rule
\begin{equation}
p(w|\mathcal{D})=\frac{p(\mathcal{D}|w) p(w)}{p(\mathcal{D})}.
\end{equation}
With posteriori (7), we can predict the label of unseen data $x$ by the conditional probability between training data and test data
\begin{equation}
p(y|x, \mathcal{D})=\int p(y|x, w) \cdot p(w|\mathcal{D}) dw.
\end{equation}
In BNNs literature [37], the entropy of distribution (8) is usually called epistemic uncertainty, which by definition, means the model's confidence about the prediction based on the knowledge learned from training data.

\par In our pFSSL paradigm, clients sharing relevant knowledge are encouraged to collaborate with each other to gain performance improvement. The epistemic uncertainty can be a powerful tool to estimate the data-relation between different clients. We assume there are $K$ clients engaged in the federated learning application; each maintains a labeled dataset denoted by $D_k^s$ and the posteriori of the model in client $k$ is denoted by $p(w_k | D_k^s)$.

\par In general, the integral part of (8) is approximated by a Monte Carlo integration. By sampling $\widehat{w}_k^t \sim p(w_k | D_k^s )$ for $T$ times, the predictive result and uncertainty of prediction on data $\widehat{x}$ can be approximated as follows:
\begin{equation}
p\left(y | \hat{x}, D_{k}^{s}\right)=\frac{1}{T} \sum_{t=1}^{T} f_{\widehat{w}_{k}^{t}}(\hat{x}),
\end{equation}
\begin{equation}
H\left(y | \hat{x}, D_{k}^{s}\right)=-\sum_{c=1}^{C} p\left(y=c | \hat{x}, D_{k}^{s}\right) \log p\left(y=c | \hat{x}, D_{k}^{s}\right),
\end{equation}
where the $f_{\widehat{w}_{k}^{t}}(\hat{x})$ is the predictive function parameterized by $\widehat{w}_{k}^{t}$ (i.e., neural network with softmax layer) and $C$ is the total number of categories in classification. To ease notation, we simplify the formulations as $P_{w_k}(y|\hat{x})=p\left(y | \hat{x}, D_{k}^{s}\right)$ and$H_{w_k}(y|\hat{x})=H\left(y | \hat{x}, D_{k}^{s}\right)$.

\par With Eqs. (9) and (10), we then design the labeling mechanism for the unlabeled dataset $D_k^u$ held by client $k$. Specifically, for a single unlabeled point $x_i \in D_k^u$, we adopt the most confident prediction from the helper agents:
\begin{equation}
\widehat{y}_{i}=\min _{H_{w_{j}}\left(y | x_{i}\right)} P_{w_{j}}\left(y | x_{i}\right), \forall w_{j} \in H_{k}.
\end{equation}

\par By applying (11) to all unlabeled points in every unlabeled dataset, we can easily obtain the requisite pseudo labels for unsupervised loss (5). The visualization of pseudo-labeling mechanism is given in Fig. \ref{Fig.2}. By minimizing the divergence between pseudo labels and the predictions of local model, these clients can learn from relevant neighbors thus to decrease the uncertainty about local unlabeled data.

\begin{figure}[htb]
   \centering
   \includegraphics[width=3in]{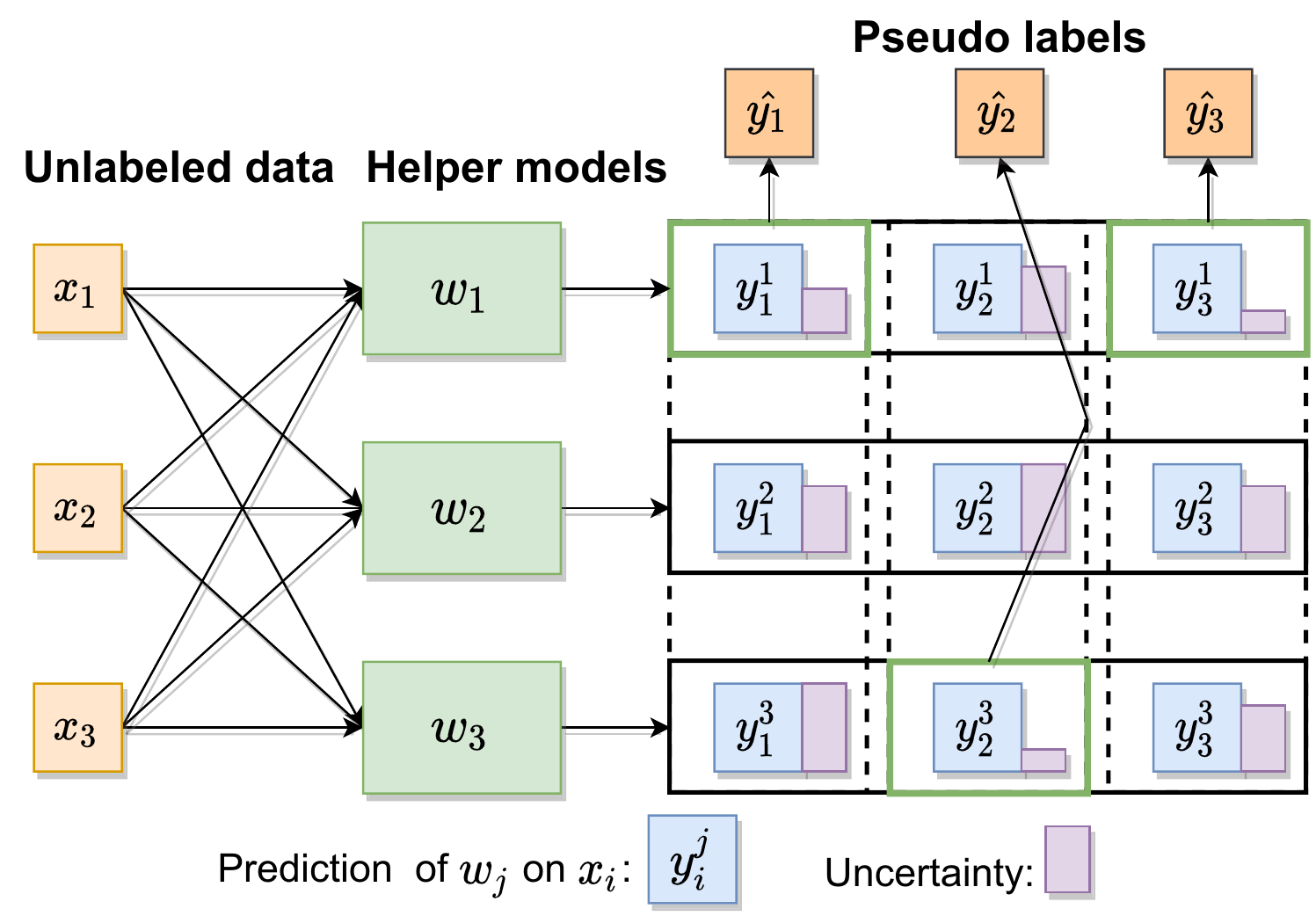}
   \caption{Illustration of pseudo-labeling mechanism.}
   \label{Fig.2}
\end{figure}

\par In this paper, we adopt the MC-dropout [38] to sample the model parameter from approximate distribution $\widehat{w}_k^t \sim p(w_k;dropout=True)$, thus, the posteriori can be transferred  without additional burden but model parameter $w_k$.

\subsection{Helper Selection}
\par As each client attempts to search its helper agents from the whole group, the amount of data that needs to be transmitted in a single communication round can be denoted as $K(K-1)|w_k|$, if $K$ and $|w_k|$ are large, the total transmitted data will be too heavy to apply the training efficiently. Thus, we design a communication-efficient helper selection protocol which can significantly decrease the network overhead of the system.
\begin{figure}[htb]
   \centering
   \includegraphics[width=3in]{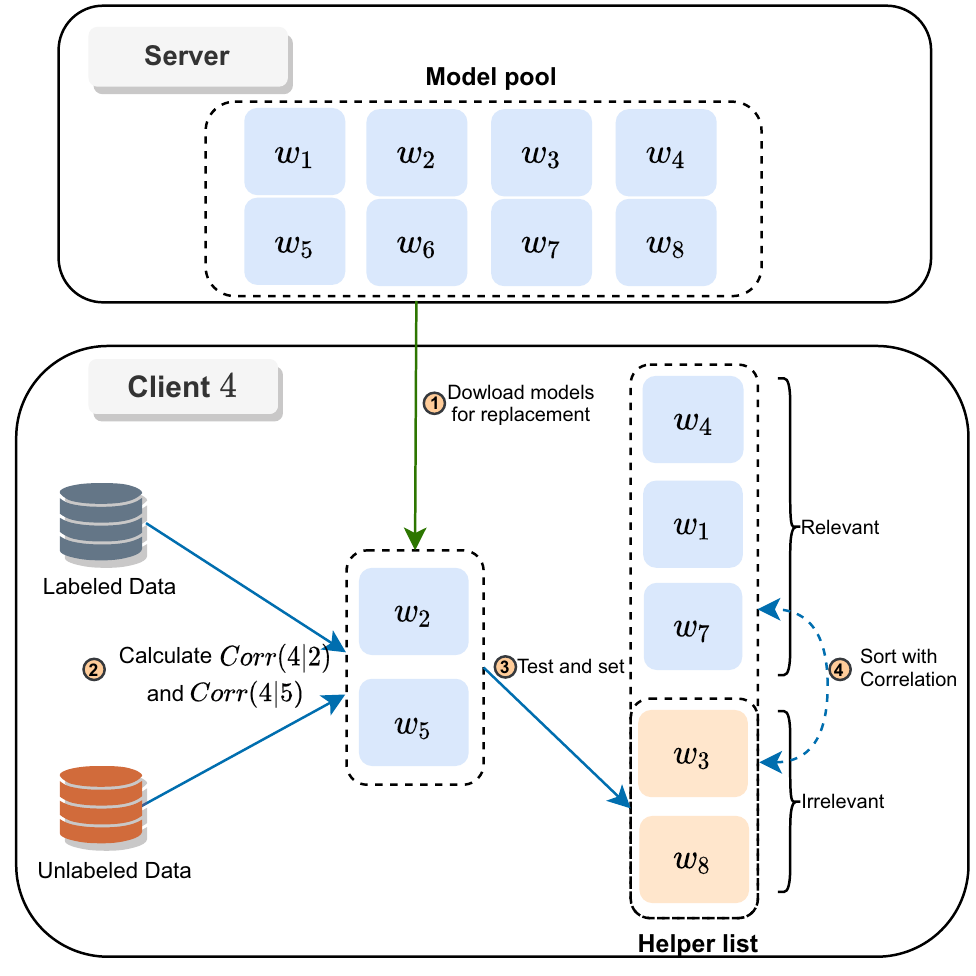}
   \caption{Illustration of helper-replace procedure.}
   \label{Fig.3}
\end{figure}
\par For a specific client $k$, instead of requesting all peer models in every communication round, we prefer to maintain a helper list with size $M$ to cache the helpers, where $M \ll K$. The cached helpers are selected according to the uncertainty-based relation among clients, in detail, we define the uncertainty of client $j$ on client $k$ as
\begin{equation}
H\left(D_{k}^{u} | D_{j}^{s}\right)=\sum_{x_{i} \in D_{k}^{u}} H_{w_{j}}\left(y | x_{i}\right),
\end{equation}
where $H_{w_j}(y | x_i)$ can be calculated by Eqs. (9) and (10), it represents the summation of uncertainties of client $j$'s model to client $k$'s unlabeled data. Then the corresponding metric can be defined as the combination of uncertainty of unlabeled data and accuracy on labeled data:

\begin{equation}
\operatorname{Corr}(k | j)=\left(1-\mu_{k}\right) \bar{H}\left(D_{k}^{u} | D_{j}^{s}\right)+\mu_{k} \operatorname{Acc}\left(D_{k}^{s} ; w_{j}\right),
\end{equation}
where $\mu_{k}$ is the ratio defined in Eq. (3) and $\bar{H}(D_k^u | D_j^s)$ is the residue of normalized entropy calculated by
\begin{equation}
\bar{H}\left(D_{k}^{u} | D_{j}^{s}\right) = 1 - H\left(D_{k}^{u} | D_{j}^{s}\right) / H(C),
\end{equation}where $H(C)$ is the maximum entropy of $C$-class prediction. We utilize this metric to measure the data-relation between clients and treat it as the weight for model aggregation. In Section \ref{sec:theory}, we give the theoretical intuition behind this metric.

\par In practice, the one-to-one coefficient values are not calculated at one-time since the representation ability of each local model can be deficient at startup rounds. Instead, we use the ranking update to iteratively select relevant helpers. For helpers in list $H_k$, the client ranks their coefficients (compute by Eq. (13)) and attempts to search substitutes for the most $R$ irrelevant helpers in current round. This searching operation is only carried out in the first $F$ rounds. Fig. \ref{Fig.3} visualizes the helper selection procedure with 8 clients, client list of length 5 and 2 helpers to replace in each replacing round. After the corresponding metric is calculated, the {\it{test and set}} stage compares this value with the most irrelevant helper in helper list. The replacement process will be performed if the newly selected model outperforms this helper. At the end of the round, the list of helpers will be sorted to select the $R$ least relevant helpers. In addition, each client updates local-stored helpers' models from server periodically (every $\nu$ rounds), and skip the helpers which has not updated itself.  Besides that, we implement a warmup training stage before the first communication round to avoid selecting helpers with untrained models.

\par Combining the labeling mechanism and helper selection, we derive the summary of our proposed method in Algorithm 1, with the subroutines depicted in Algorithms 2 and 3.

\begin{algorithm}[htp]
 \renewcommand{\algorithmicrequire}{\textbf{Require:}}
 \renewcommand{\algorithmicensure}{\textbf{Output:}}
 \renewcommand{\algorithmicreturn}{\textbf{begin}}
\caption{Uncertainty Minimization Algorithm}
\begin{algorithmic}[1]
\vspace{.2cm}
\Require
Number of clients: $K$; Number of sampled clients in each round: $\bar{K}$; Length of helper list: $M$; Total rounds: $T$; Number of local epoches: $E$; Client set: $\left\{w_{k}\right\}_{k=1}^{K}$; Unlabeled datasets: $\left\{D_{k}^{u}\right\}_{k=1}^{K}$; Labeled datasets: $\left\{D_{k}^{s}\right\}_{k=1}^{K}$; Learning rate: $\mu$; The fequence of helper update: $\nu$; Helper searching rounds: $F$.

\State Initializing each client's model $w_k$ with $w_0$;
\State Initialize each client's helper list with list $H_k:\{w_k\}$;
\For {each client $k \in [K]$}
	\State local warmup training;
\EndFor
\For {${t}=1$ to $T$}
	\If {$t < F$}
		\For {each client $w_k \in \left\{w_{k}\right\}_{k=1}^{K}$}
			\State \Call{ReplaceHelper}{$w_k,H_k$};
		\EndFor
	\EndIf
	\If {$t\% \nu \equiv 0$}
		\For {each client $w_k \in \left\{w_{k}\right\}_{k=1}^{K}$}
			\State \Call{UpdateHelper}{$w_k,H_k$};
		\EndFor
    \EndIf
	\State Sample $\bar{K}$ training client  $\left\{w_{k}\right\}_{k=1}^{\bar{K}}$;
	\For {each client $w_k \in \left\{w_{k}\right\}_{k=1}^{\bar{K}}$}
		\If {helper list $H_k$ is not full}
			\State Randomly select $M-|H_k|$ helper to fill $H_k$;
		\EndIf

		\State Local parameter aggregation:
		\State \hspace{0.5cm} $w_{k} \leftarrow \frac{\sum_{w_{j} \in H_{k}} \operatorname{Corr}(k / j) * w_{j}}{\sum_{w_{j} \in H_{k}} \operatorname{Corr}(k / j)}$;
		\For {$x_{i} \in D_{k}^{u}$}
			\State $\widehat{y}_{i}=\min _{H_{w_{j}}\left(y \mid x_{i}\right)} P_{w_{j}}\left(y \mid x_{i}\right), \forall w_{j} \in H_{k}$;
		\EndFor
      	\For {$e=1$ to $E$}
			\State $w_{k} \leftarrow w_{k}-\mu \nabla_{w_{k}} \mathcal{L}^{s}\left(D_{k}^{s}, w_{k}\right)$;
			\State $w_{k} \leftarrow w_{k}-\mu \nabla_{w_{k}} \sum_{x_{i} \in D_{k}^{u}} K L\left(f_{w_{k}}\left(x_{i}\right), \widehat{y}_{i}\right)$;
		\EndFor
	 \EndFor
	 \State Upload the $\left\{w_{k}\right\}_{k=1}^{\bar{K}}$ to central server.
\EndFor
\State \bf{Training Stop}
\end{algorithmic}
\end{algorithm}

\begin{algorithm}[htp]
 \renewcommand{\algorithmicrequire}{\textbf{Require:}}
 \renewcommand{\algorithmicensure}{\textbf{Output:}}
 \renewcommand{\algorithmicreturn}{\textbf{begin}}
 \renewcommand{\algorithmicreturn}{\textbf{begin}}
\caption{UpdateHelper function for Algorithm 1}
\begin{algorithmic}[1]
\vspace{.2cm}
\Function {UpdateHelper}{$w_k,H_k$}
\State Create an update list $Up: \{\}$;
\For {$w_j  \in H_k$}
	\State Compute $\operatorname{Corr}(k | j)$ with Eq. (13);
	\If {$\operatorname{Corr}(k|j)$ not rank the lowest $R$ {\bf{or}} $j \equiv  k$}
		\State Push $w_j$ to $Up$;
	\EndIf
\EndFor
\For {$w_j \in Up$}
	\State Check update record of $w_j$ in server;
	\If {$w_j$ has been update}
		\State Download updated $w_j$;
	\EndIf
\EndFor
\State $H_{k} \leftarrow Up \cup H_{k}$.
\EndFunction
\State \bf{endfunction}
\end{algorithmic}
\end{algorithm}

\begin{algorithm}[htp]
 \renewcommand{\algorithmicrequire}{\textbf{Require:}}
 \renewcommand{\algorithmicensure}{\textbf{Output:}}
 \renewcommand{\algorithmicreturn}{\textbf{begin}}
 \renewcommand{\algorithmicreturn}{\textbf{begin}}
\caption{ReplaceHelper function for Algorithm 1}
\begin{algorithmic}[1]
\vspace{.2cm}
\Function {ReplaceHelper}{$w_k,H_k$}
\State Create an update list $Re: \{\}$;
\For {$w_j  \in H_k$}
	\State Compute $\operatorname{Corr}(k | j)$ with Eq. (13);
	\If {$\operatorname{Corr}(k|j)$ rank the lowest $R$ {\bf{and}} $j \neq  k$}
		\State Push $w_j$ to $Re$;
	\EndIf
\EndFor
\State Randomly download $R$ models  $\left\{w_{r}\right\}_{r=1}^{R} \cap H_k = \emptyset$ from server;
\For {$w_r \in \left\{w_{r}\right\}_{r=1}^{R}$}
	\State Compute $\operatorname{Corr}(k|r)$ with Eq. (13);
	\State Replace $w_j$ in $Re$ with $w_r$ {\bf{if}}  $\operatorname{Corr}(k|r) > \operatorname{Corr}(k|j)$;
\EndFor
\State $H_{k} \leftarrow Re \cup H_{k}$.
\EndFunction
\State \bf{endfunction}
\end{algorithmic}
\end{algorithm}

\section{Theoretical Analysis}\label{sec:theory}
\par In this section, we utilize the theories from domain adaptation [39] to prove that the UM-pFSSL can improve the generalization bound of the aggregated model.

\par Consider the FL setting with $K$ clients, each client maintains the local domain $\mathcal{T}_{k}:\left\langle\mathcal{D}_{k}, f_{k}^{*}\right\rangle$ with $D_k=D_k^s \cup D_k^u$, where $D_k^s$ and $D_k^u$ are samples from personalized data distribution $P_{X,Y}^k$ and its marginal distribution $P_X^k$, respectively. $f_k^{*}$ denotes the labeling function that maps data $x \sim P_X^k$ to its true label. We treat client $k$ as the target domain and its helpers as the source domains in domain adaptation. Then the source domains can be defined as $\mathcal{T}_{j}:\left\langle\mathcal{D}_{j}, f_{j}^{*}\right\rangle_{j \in[M]}$, where $D_j$ is the local data of helper unknown to client $k$ and $M$ is the length of the helper list.

\begin{lemma}[{[40]}]
Suppose source $\mathcal{T}_j$ has m labeled instances. $\mathcal{H}$ denotes a hypothesis space of VC-dimension $d$. Then for any $\delta \in (0,1)$, with probability at least $1-\delta$, for every $h \in \mathcal{H}$:
\begin{equation}
\epsilon_{k}(h) \leq \epsilon_{j}(h)+\frac{1}{2} d_{\mathcal{H} \Delta \mathcal{H}}\left(P_{X}^{k}, P_{X}^{j}\right)+\lambda_{j},
\end{equation}
where $\lambda_{j}=\min _{h \in \mathcal{H}} \epsilon_{k}(h)+\epsilon_{j}(h)$; $\epsilon_{k}(h)$ denotes the ideal risk of hypothesis $h$ on domain $\mathcal{T}_k$. $d_{\mathcal{H} \Delta \mathcal{H}}\left(P_{X}^{k}, P_{X}^{j}\right)$ is the divergence of two data distributions measured in a symmetric-difference hypothesis space.
\end{lemma}
\par This lemma implies that, the performance of hypothesis $h$ (which is well-behaved on $P_{X}^{j}$) on distribution $P_{X}^{k}$ depends on the divergence of the two distirbution $d_{\mathcal{H} \Delta \mathcal{H}}\left(P_{X}^{k}, P_{X}^{j}\right)$. In other words, when the divergence between the two datasets is small, a model that performs well on one dataset can also achieve good results on the other dataset.
\begin{corollary}
 Define $\hat{h}=\frac{1}{M} \sum_{j=1}^{M} h_{j}$  as the aggregated hypothesis of helpers (include $h_k$ itself), Then for any $\delta \in (0,1)$, with probability at least $1-\delta$, for every $h \in \mathcal{H}$:
\begin{equation}
\epsilon_{k}(\hat{h}) \leq \frac{1}{M} \sum_{j=1}^{M} \epsilon_{j}\left(h_{j}\right)+\frac{1}{2 M} \sum_{j=1}^{M}\left(d_{\mathcal{H} \Delta \mathcal{H}}\left(P_{X}^{k}, P_{X}^{j}\right)+2 \lambda_{j}\right).
\end{equation}
\end{corollary}
\begin{proof}
We follow the method from [36]. Based on Lemma 1, when client $j$ is the source and client $k$ is the target, then for $\forall \frac{\delta}{M}>0$, with probability $1-\frac{\delta}{M}$:
\begin{equation}
\epsilon_{k}\left(h_{j}\right) \leq \epsilon_{j}\left(h_{j}\right)+\frac{1}{2} d_{\mathcal{H} \Delta \mathcal{H}}\left(P_{X}^{k}, P_{X}^{j}\right)+\lambda_{j}.
\end{equation}
\par The convexity of empirical estimation makes the Jenson inequality hold:
\begin{equation}
\epsilon_{k}(\hat{h})=\epsilon_{k}\left(\frac{1}{M} \sum_{j=1}^{M} h_{j}\right) \leq \frac{1}{M} \sum_{j=1}^{M} \epsilon_{k}\left(h_{j}\right),
\end{equation}
Thus,
\begin{equation}
\begin{aligned}
& \operatorname{Pr}\Bigg\{\epsilon_{k}(\hat{h})>\frac{1}{M} \sum_{j=1}^{M}\left(\epsilon_{j}\left(h_{j}\right)+\frac{1}{2} d_{\mathcal{H} \Delta \mathcal{H}}\left(P_{X}^{k}, P_{X}^{j}\right)+\lambda_{j}\right)\Bigg\} \\
&\begin{aligned}
\leq \operatorname{Pr}\Bigg\{ \frac{1}{M} & \sum_{j=1}^{M} \epsilon_{k}\left(h_{j}\right) > \\
& \frac{1}{M} \sum_{j=1}^{M}\left(\epsilon_{j}\left(h_{j}\right)+\frac{1}{2} d_{\mathcal{H} \Delta \mathcal{H}}\left(P_{X}^{k}, P_{X}^{j}\right)+\lambda_{j}\right)\Bigg\}
\end{aligned}
\\
&\begin{aligned}
\leq \operatorname{Pr}\Bigg\{\bigvee_{j \in[M]} \epsilon_{k}\left(h_{j}\right)>\epsilon_{j}\left(h_{j}\right)+\frac{1}{2} d_{\mathcal{H} \Delta \mathcal{H}}\left(P_{X}^{k}, P_{X}^{j}\right)+\lambda_{j}\Bigg\}
\end{aligned}
\\
&\begin{aligned}
\leq  \sum_{j=1}^{M} \frac{\delta}{M}=\delta.
\end{aligned}
\end{aligned}
\end{equation}
\end{proof}

\begin{theorem}
For a FL system with $K$ clients, select $M$ helper clients to build aggregated hypothesis $\hat{h}=\frac{1}{M} \sum_{j=1}^{M} h_{j}$. There are two generalization bounds for applying $\hat{h}$ to domain $\mathcal{T}_k$:
\begin{equation}
\begin{aligned}
\epsilon_{k}(\hat{h}) \leq & \frac{1}{M} \sum_{j=1}^{M}\left(\epsilon_{j}\left(h_{j}\right)+\frac{1}{2} \hat{d}_{\mathcal{H} \Delta \mathcal{H}}\left(\mathcal{D}_{k}, \mathcal{D}_{j}\right)+\lambda_{j}\right)\\
& +4 \sqrt{\frac{2 d \log \left(2em^{*}\right)+\log \left(\frac{4 M}{\delta}\right)}{m^{*}}},
\end{aligned}
\end{equation}
\begin{equation}
\epsilon_{k}(\hat{h}) \leq \frac{1}{M} \sum_{j=1}^{M} \hat{\epsilon}_{k}\left(h_{j}\right)+\sqrt{\frac{4 d \log \left(\frac{2 e m}{d}\right)+\log \left(\frac{4 M}{\delta}\right)}{m}},
\end{equation}
where $\hat{\epsilon}_j$ and $\hat{\epsilon}_k$ are empirical risk on domain $\mathcal{T}_k$ and $\mathcal{T}_j$, respectively. $m$ denotes the sample size of labeled data. $m^{*}$ is the size of $D_k$ and $D_j$ and $e$ is the natural constant. $\hat{d}_{\mathcal{H} \Delta \mathcal{H}}\left(\mathcal{D}_{k}, \mathcal{D}_{j}\right)$ is the empirical estimate of $d_{\mathcal{H} \Delta \mathcal{H}}\left(P_{X}^{k}, P_{X}^{j}\right)$.
\end{theorem}

\begin{proof}
Based on Vapnik-Chervonenkis theory:
\begin{equation}
\epsilon_{k}\left(h_{j}\right) \leq \hat{\epsilon}_{k}\left(h_{j}\right)+\sqrt{\frac{4 d \log \left(\frac{2 e m}{d}\right)+\log \left(\frac{4}{\delta}\right)}{m}}.
\end{equation}
\par Further, we adopt the derivation in [40]:
\begin{equation}
\begin{aligned}
\frac{1}{2} d_{\mathcal{H} \Delta \mathcal{H}}\left(P_{X}^{k}, P_{X}^{j}\right) & \leq \frac{1}{2} \hat{d}_{\mathcal{H} \Delta \mathcal{H}}\left(\mathcal{D}_{k}, \mathcal{D}_{j}\right) \\
& +4 \sqrt{\frac{2 d \log \left(2 em^{*}\right)+\log \left(\frac{4 M}{\delta}\right)}{m^{*}}}.
\end{aligned}
\end{equation}
\par According to Corollary 1:
\begin{equation}
\begin{aligned}
\epsilon_{k}(\hat{h}) \leq & \frac{1}{M} \sum_{j=1}^{M} \epsilon_{k}\left(h_{j}\right) \\
\leq & \frac{1}{M} \sum_{j=1}^{M} \epsilon_{j}\left(h_{j}\right)\\
& +\frac{1}{2 M} \sum_{j=1}^{M}\left(d_{\mathcal{H} \Delta \mathcal{H}}\left(P_{X}^{k}, P_{X}^{j}\right)+2 \lambda_{j}\right).
\end{aligned}
\end{equation}
\par Thus, one can derive
\begin{equation}
\begin{aligned}
\epsilon_{k}(\hat{h})  \leq & \frac{1}{M} \sum_{j=1}^{M} \epsilon_{k}\left(h_{j}\right) \\
\leq &  \frac{1}{M} \sum_{j=1}^{M} \epsilon_{j}\left(h_{j}\right)+\frac{1}{2 M} \sum_{j=1}^{M}\left(d_{\mathcal{H} \Delta \mathcal{H}}\left(P_{X}^{k}, P_{X}^{j}\right)+2 \lambda_{j}\right) \\
\leq & \frac{1}{M} \sum_{j=1}^{M}\left(\epsilon_{j}\left(h_{j}\right)\right.\left.+\frac{1}{2} \hat{d}_{\mathcal{H} \Delta \mathcal{H}}\left(\mathcal{D}_{k}, \mathcal{D}_{j}\right)+\lambda_{j}\right) \\
& +4 \sqrt{\frac{2 d \log \left(2 e m^{*}\right)+\log \left(\frac{4 M}{\delta}\right)}{m^{*}}},
\end{aligned}
\end{equation}
and
\begin{equation}
\begin{aligned}
\epsilon_{k}(\hat{h}) & \leq \frac{1}{M} \sum_{j=1}^{M} \epsilon_{k}\left(h_{j}\right) \\
& \leq \frac{1}{M} \sum_{j=1}^{M} \hat{\epsilon}_{k}\left(h_{j}\right)+ \sqrt{\frac{4d \log \left( \frac{2em}{d}\right)+\log \left(\frac{4 M}{\delta}\right)}{m}}.
\end{aligned}
\end{equation}
Therefore, the formulations (25) and (26) represent the two bounds in Theorem 1.
\end{proof}
\par Then we have two generalization bounds for aggregated hypothesis: Eqs. (20) and (21), we proximate the $\hat{d}_{\mathcal{H} \Delta \mathcal{H}}\left(\mathcal{D}_{k}, \mathcal{D}_{j}\right)$ with $\bar{H}\left(\mathcal{D}_{k}|\mathcal{D}_{j}\right)$ in (14) based on the intuitions:
\begin{itemize}
 	\item[1)] If $D_k$ has no knowledge related to $D_j$, then $\bar{H}\left(\mathcal{D}_{k}|\mathcal{D}_{j}\right)=0$.
	\item[2)] If $D_k=D_j$, then $\bar{H}\left(\mathcal{D}_{k}|\mathcal{D}_{j}\right)=1$.
     \item[3)] If $D_k$  is partial related to $D_j$, then $\bar{H}\left(\mathcal{D}_{k}|\mathcal{D}_{j}\right)$ locates to interval $(0,1)$. Larger $\bar{H}\left(\mathcal{D}_{k}|\mathcal{D}_{j}\right)$ means stronger relation between $D_k$ and $D_j$.
\end{itemize}
\par For (21), when the labeled dataset is large, the generalization risk is bound by $O\left(\sqrt{\frac{\mathrm{d}}{m} \log \left(\frac{m}{d}\right)-\frac{1}{m} \log (\delta)}\right) \approx 0$, thus the empirical risk $\hat{\epsilon}_k(h_j)$ is sufficient to measure the generalization bound of $\hat{h}$. When $m$ is small (close to 0), then second term of Eq. (21) could be too large thus eliminate the impact of empirical risk, in this scenario, the Eq. (20) is more suitable to evaluate the generalization bound. These intuitions encourage us to design a metric to measure data-relation between clients with both their labeled and unlabeled data, which is formulated as Eq. (13).

\section{Experiment Results and  Evaluation}\label{sec:result}
\par In this section, the experiments are performed to compare the performance of our proposed scheme with other key related works under the heterogeneous federated semi-supervised learning context. The results demonstrate the superiority of our method in several benchmarks across clients.
\subsection{Experimental Setup}
\par \textbf{Datasets and Non-IID setting:}  We mainly evaluate our proposed scheme with two different datasets, which are Fashion-MNIST\footnote{https://github.com/zalandoresearch/fashion-mnist} and CIFAR-10\footnote{https://www.cs.toronto.edu/\%7ekriz/cifar.html}. CIFAR-10 consists of $32\times32$ RGB images while Fashion-MNIST consists of $28\times28$ grayscale images. All of the datasets are labeled with 10 classes. We randomly split the datasets with 70\%, 10\% and 20\% for training, validation and testing, respectively. We adopt the same Non-IID setting from existing work [41], which utilizes a Dirichlet distribution ${Dir}_K(\alpha)$ to assign the training data and testing data to different clients. In detail, for each class c, we sample a $K$-dimensional variable $p_c$ from ${Dir}_K(\alpha)$ which is used to allocate different proportions of samples of class $c$ to $K$ clients. The parameter $\alpha$ represents the heterogeneous degree of potential data distributions of clients, with smaller $\alpha$ indicates higher heterogeneity. For training data on each client, we further sample a $2$-dimensional variable $p_k:=(p_k^s,p_k^u)$ from ${Dir}_2(0.5)$ to split the datasets with labeled-to-unlabeled ratio $p_k^s/p_k^u$. In our experiments, the testing data and validation data on each client share the same label distribution and labeled-to-unlabeled ratio with training data, while these settings between clients are potentially different. All results presented in this section are evaluated on the personal test sets.

\par \textbf{Model setting:} We mainly use the lightweight neural network structure ResNet-9 in our experiments. To integrate MC-dropout to this structure, we append dropout layers (with dropout rate 0.5) after last two batch normalization layers and all fully connected layers except output layer.

\par \textbf{Implementation and configuration:} We set the size of helper list, number of replaced helpers and the number of rounds to implement searching to 5, 2 and 30 respectively ($M=5,R=2,F=30$). In each round, we sample 10\% from engaged 100 clients to implement training procedure (sample rate $\tau=0.1$). The training batch for each client is set to 64, and we utilize stochastic gradient descent (SGD) optimizer with learning rate 1e-4 and momentum 0.9. The total number of training round is set to 200 and the number of local training epochs is set to 5. The helpers' models are set to update every 10 rounds ($\nu=10$).

\par All the experiments are implemented by Pytorch and trained on Nvidia GeForce RTX 3090 GPU with 24 GB GDDR6X memory.

\par \textbf{Compared methods:} To validate our method, we compare the performance of UM-pFSSL with several key related methods for Non-IID federated learning: 1) FedProx-SL: FedProx [7] with fully labeled data, 2) FedBN-SL: FedBN [17] with fully labeled data, 3) FixProx: FedProx with semi-supervised pseudo-labeling method FixMatch [31] and 4) FixBN: FedBN with FixMatch. Further, we take the FedAvg-SL and FixAvg (FedAvg with FixMatch) as the baseline methods. Besides the Non-IID methods, we utilize personalized federated learning methods combined with FixMatch for comparison: FedPer [9], LG-FedAvg [22] and pFedMe [11]. The implementations of compared methods are based on the open-source codes of original papers.

\subsection{Results and Comparisons}
\par In this subsection, we compare the results of UM-pFSSL with Non-IID schemes and personalized federated learning methods. We run each experiment 20 times and take the average as final result.
\begin{table}[htbp]
   \centering
   \caption{THE COMPARISION OF BEST TEST ACCURACY ON Fashion-MNIST WITH DIFFERENT $\alpha$.}
   \setlength{\tabcolsep}{3mm}{\begin{tabular}{c|ccccc}
      \toprule
      ~ & \textbf{Algorithms} & $\alpha$=0.5  & $\alpha$=1  & $\alpha$=5  & $\alpha$=10\\
      \midrule
      \multirow{3}*{\rotatebox{90}{\textbf{Sup.}}} & FedAvg-SL & 81.71 & 82.09	& 81.98 & 82.32   \\
      ~ & FedProx-SL & 81.83 & 83.01 & 82.65 & 83.63   \\
      ~ & FedBN-SL & 75.79 & 79.91 & 82.85 & 82.18  \\
     \midrule
	  \multirow{7}*{\rotatebox{90}{\textbf{Semi-Sup.}}} & FixAvg & 75.94 & 76.78 & 78.68 & 79.59     \\
      ~ & FixProx & 76.77 & 77.46 & 77.64 & 78.00    \\
      ~ & FixBN  & 76.40 & 79.22 & 80.48 & \textbf{81.88} \\
     \cmidrule{2-6}
      ~ & FedPer & 60.65 & 50.90 & 37.89 & 38.05     \\
	  ~ & LG-FedAvg & 74.64 & 78.49 & 75.39 & 66.54\\
      ~ & pFedMe & 44.53 & 49.01 & 47.38 & 45.39\\
     \cmidrule{2-6}
     ~ & UM-pFSSL &  \textbf{79.00} & \textbf{80.93} & \textbf{81.16} & 81.49\\
      \bottomrule
   \end{tabular}}
   \label{Tab.1}
	
\end{table}

\begin{table}[htbp]
   \centering
   \caption{THE COMPARISION OF BEST TEST ACCURACY ON CIFAR-10 WITH DIFFERENT $\alpha$.}
   \setlength{\tabcolsep}{3mm}{\begin{tabular}{c|ccccc}
      \toprule
      ~ & \textbf{Algorithms}              & $\alpha$=0.5  & $\alpha$=1  & $\alpha$=5  & $\alpha$=10\\
      \midrule
      \multirow{3}*{\rotatebox{90}{\textbf{Sup.}}} & FedAvg-SL & 48.56 & 48.25 & 54.53 & 55.33   \\
      ~ & FedProx-SL & 51.62 & 53.47 & 55.06 & 55.32   \\
      ~ & FedBN-SL & 51.54 & 52.47 & 54.32 & 55.61  \\
     \midrule
	  \multirow{7}*{\rotatebox{90}{\textbf{Semi-Sup.}}} & FixAvg & 46.41 & 46.79 & 52.46 & 53.04  \\
      ~ & FixProx &49.37 & 49.81 & 48.61 & \textbf{54.65}    \\
      ~ & FixBN  & 47.54 & 50.28 & 51.52 & 51.93\\
     \cmidrule{2-6}
      ~ & FedPer & 42.16 & 30.02 & 18.94 & 18.68     \\
	  ~ & LG-FedAvg & 43.41	 & 41.93 & 38.23 & 38.37 \\
      ~ & pFedMe & 25.70 & 27.38 & 22.56 & 18.00\\
     \cmidrule{2-6}
     ~ & UM-pFSSL & \textbf{51.14} & \textbf{52.24} & \textbf{52.83} & 54.03\\
      \bottomrule
   \end{tabular}}
   \label{Tab.2}
\end{table}

\begin{figure*}[htb]
   \centering
   \subfigure[$\alpha=0.5$]{\includegraphics[width=1.7in]{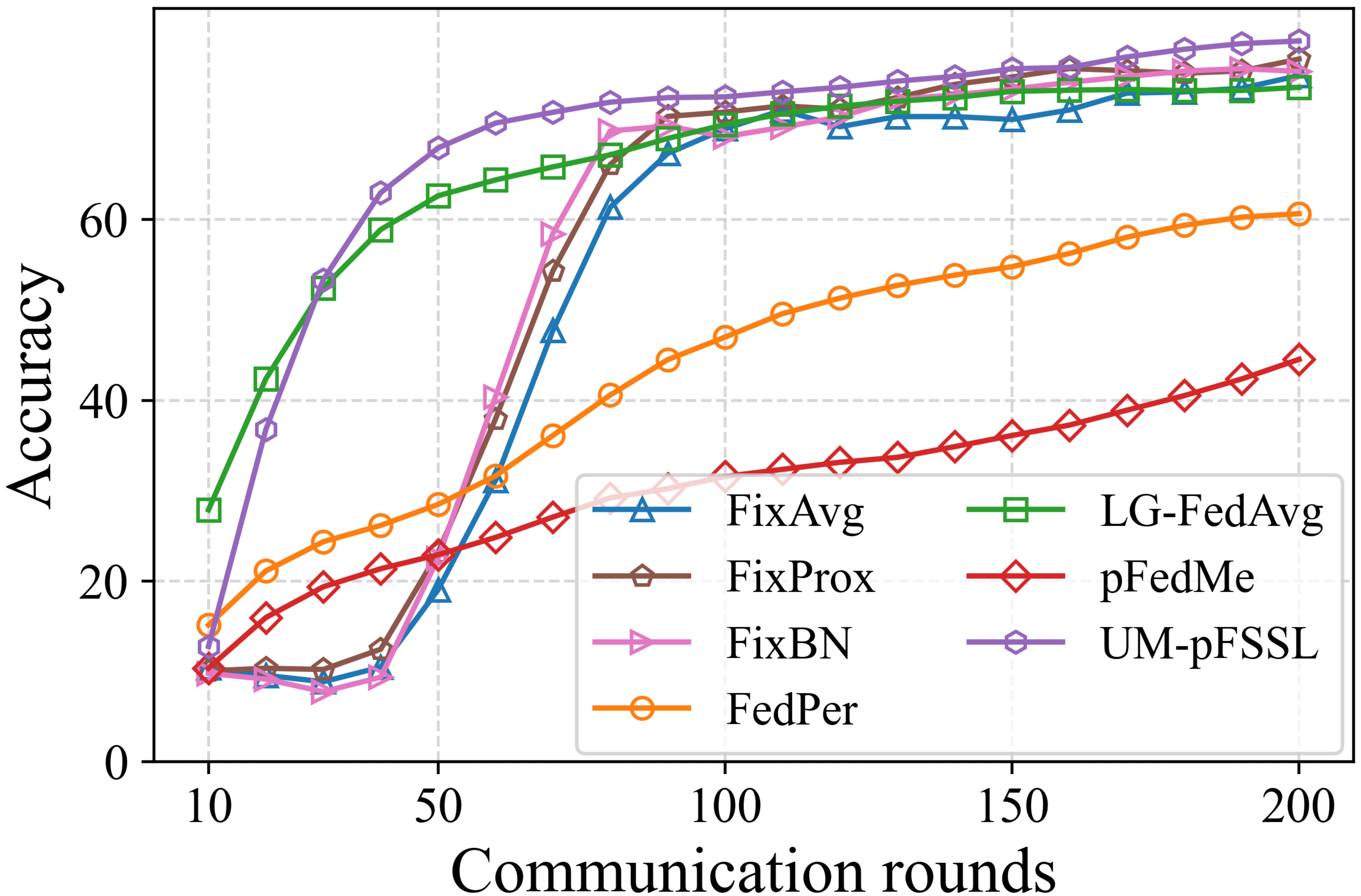}}
   \subfigure[$\alpha=1$]{\includegraphics[width=1.7in]{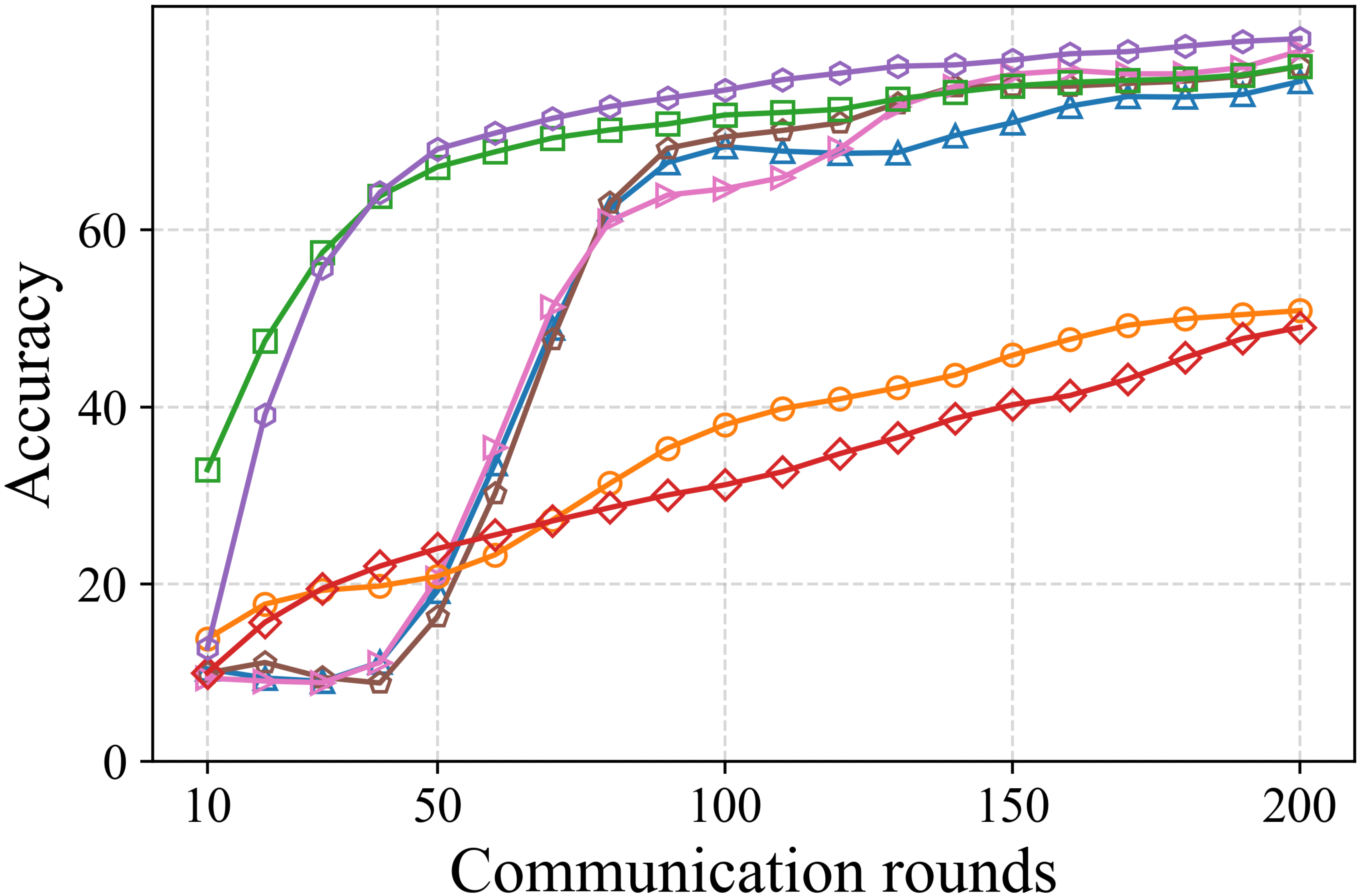}}
   \subfigure[$\alpha=5$]{\includegraphics[width=1.7in]{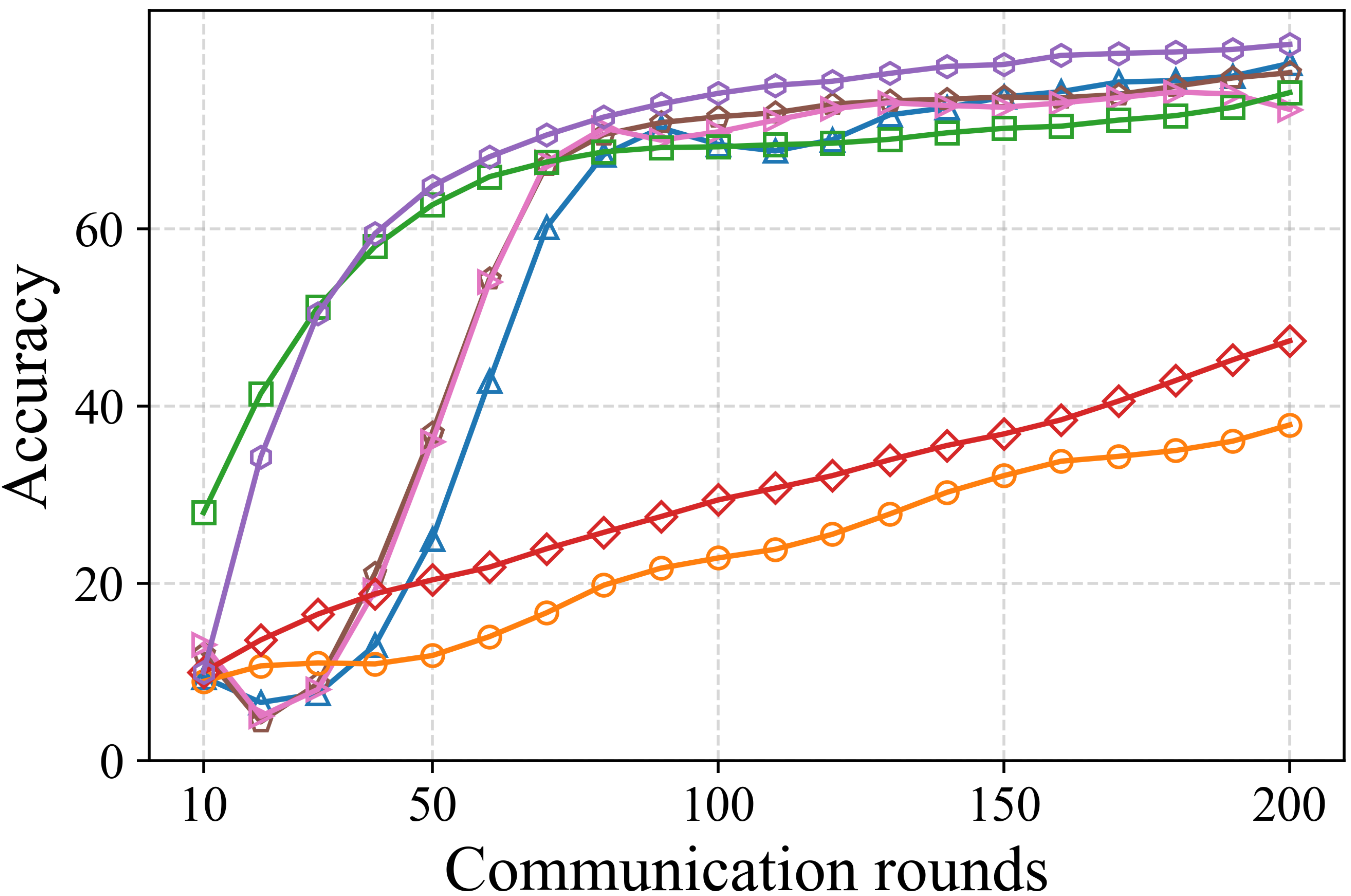}}
   \subfigure[$\alpha=10$]{\includegraphics[width=1.7in]{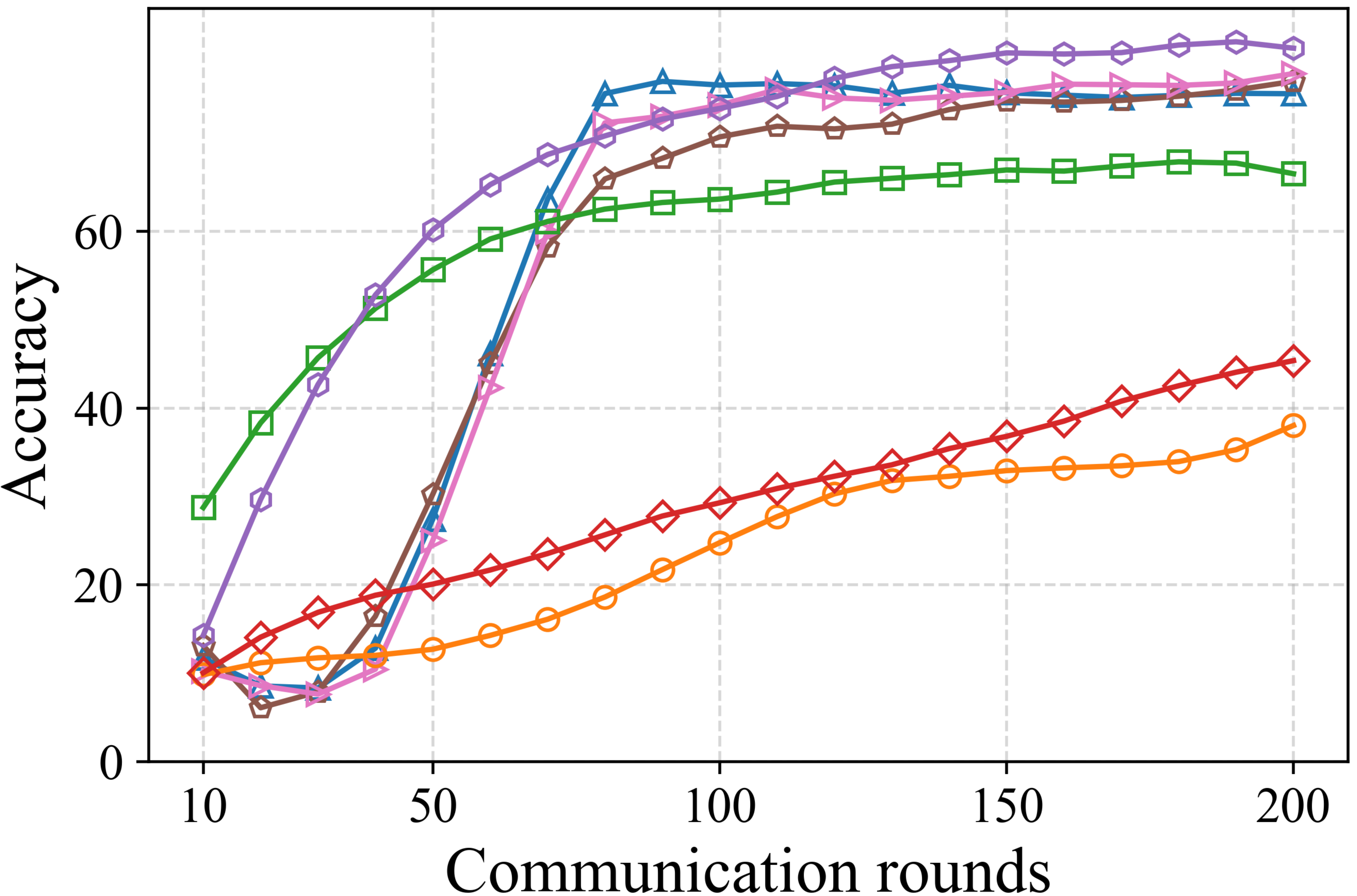}}
   \caption{Mean validation accuracy comparison of different methods on Fashion-MNIST.}
   \label{Fig.4}
\end{figure*}

\begin{figure*}[htb]
   \centering
   \subfigure[$\alpha=0.5$]{\includegraphics[width=1.7in]{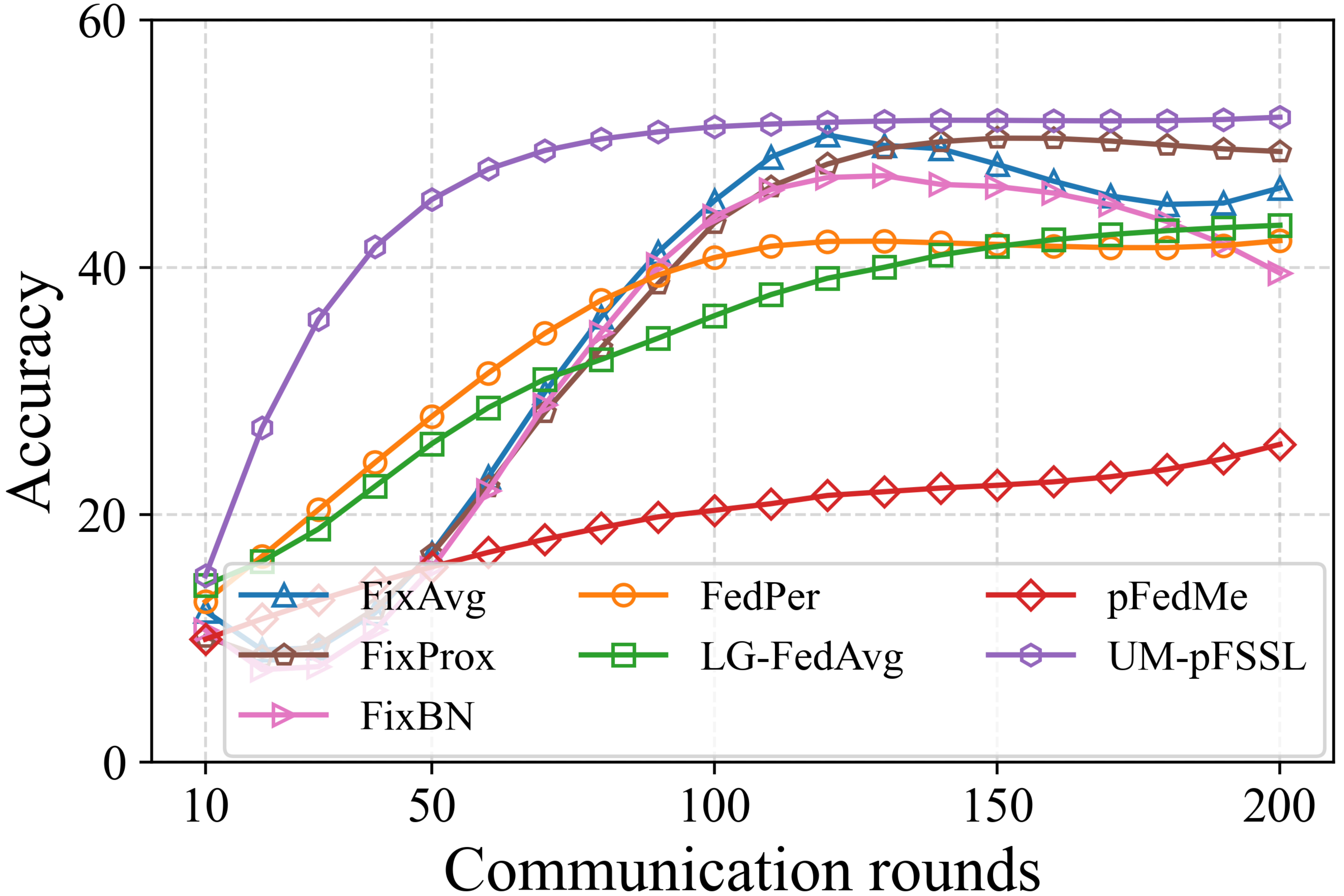}}
   \subfigure[$\alpha=1$]{\includegraphics[width=1.7in]{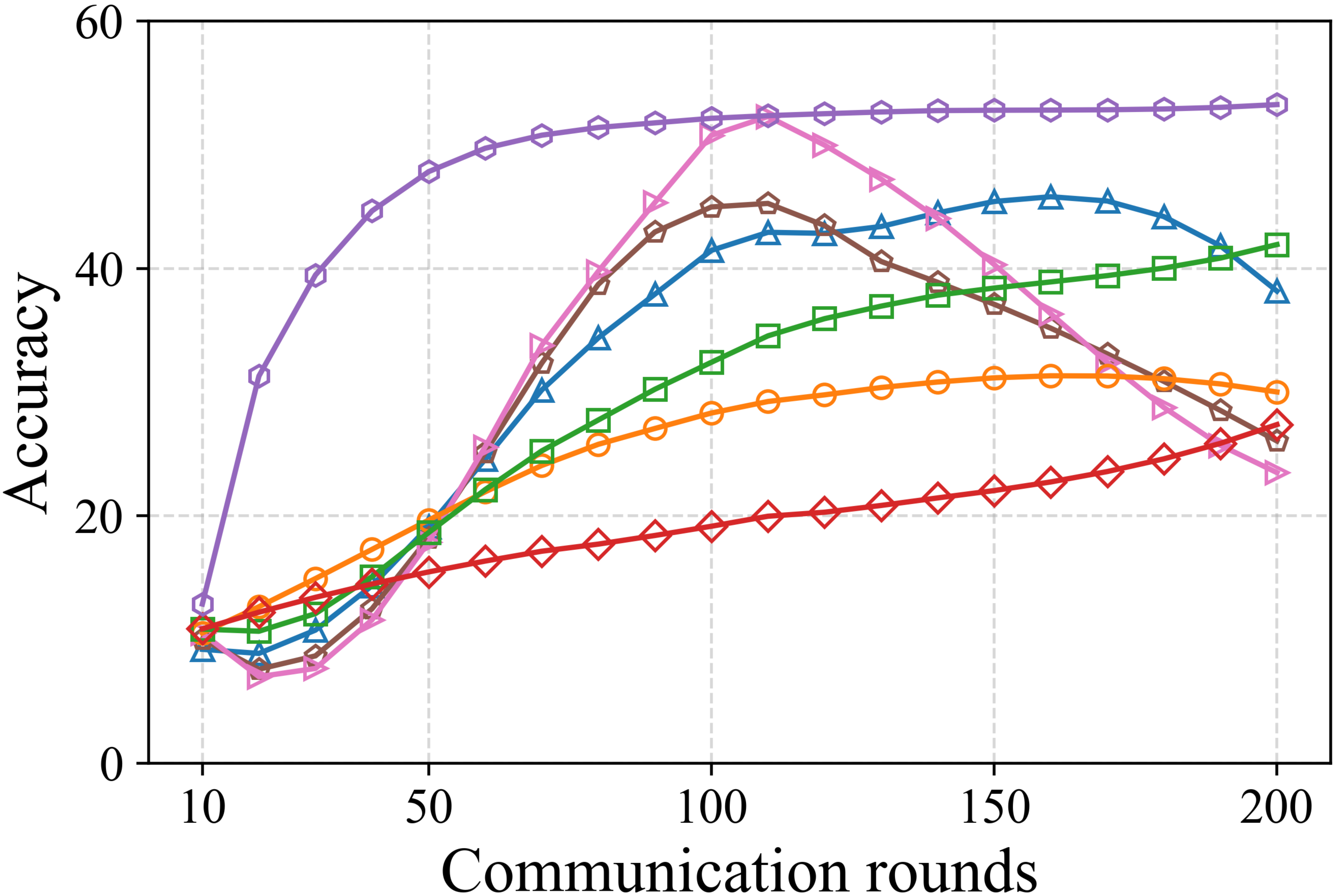}}
   \subfigure[$\alpha=5$]{\includegraphics[width=1.7in]{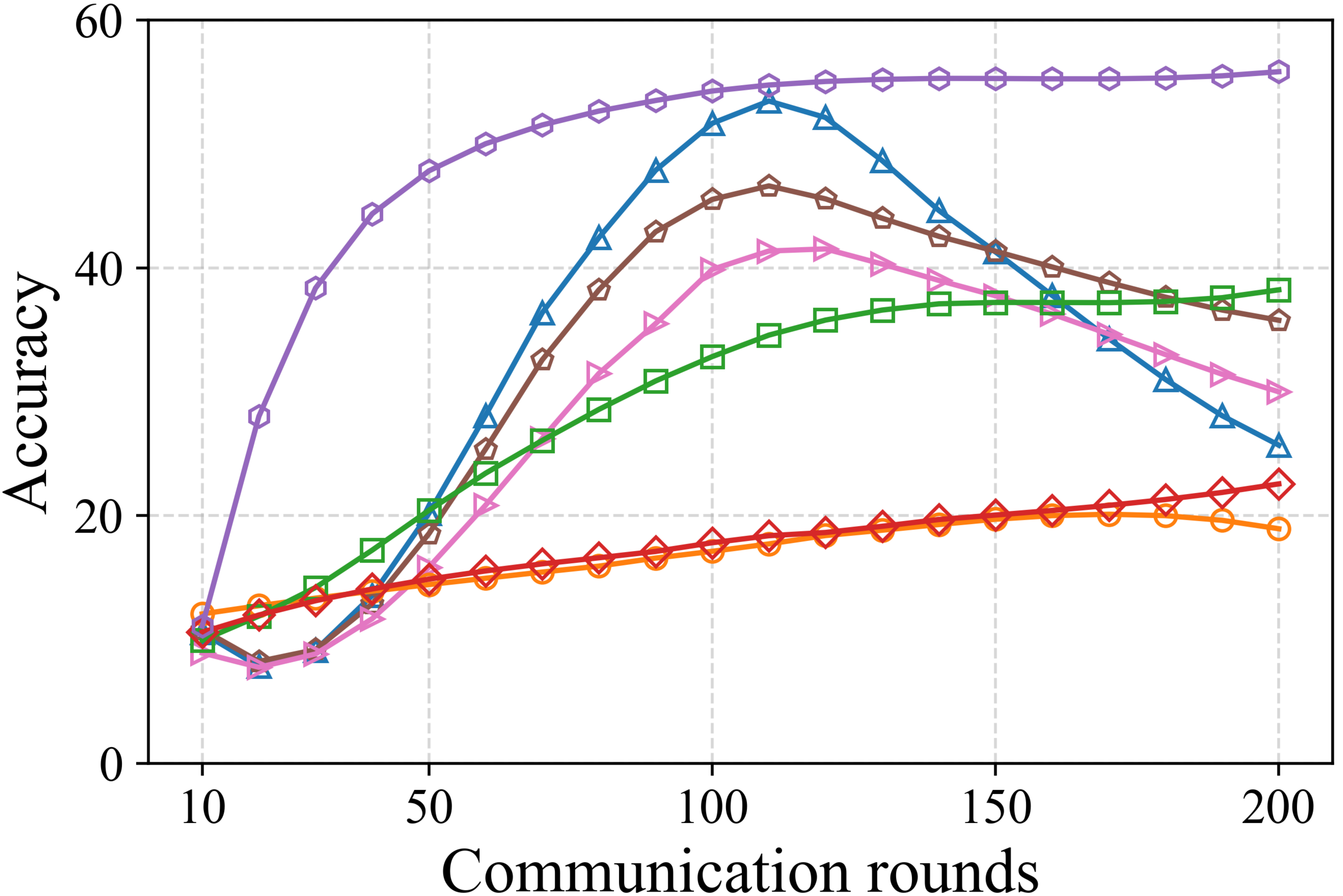}}
   \subfigure[$\alpha=10$]{\includegraphics[width=1.7in]{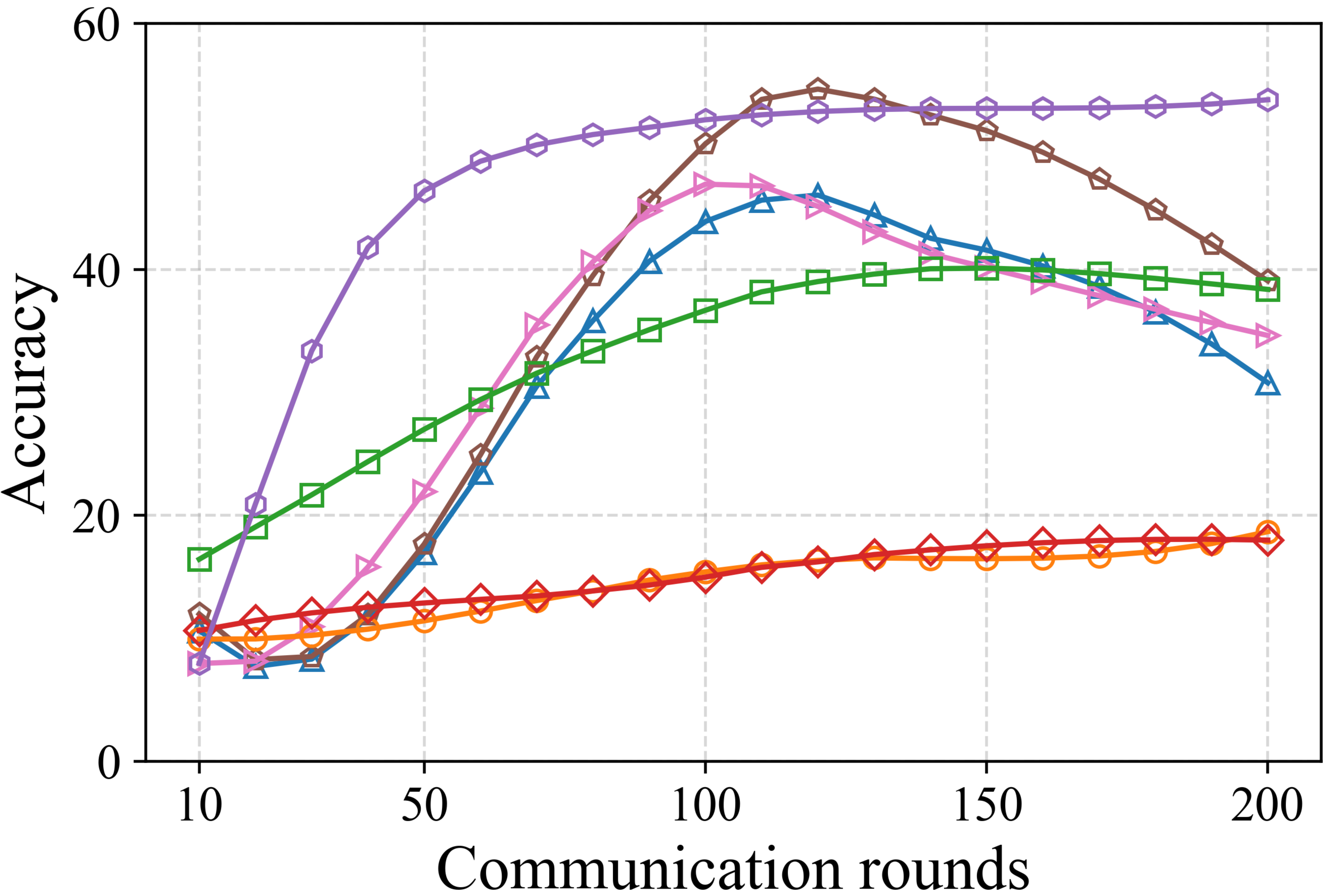}}
   \caption{Mean validation accuracy comparison of different methods on CIFAR-10.}
   \label{Fig.5}
\end{figure*}

\par \textbf{Impact of heterogeneity:} TABLE \ref{Tab.1} and TABLE \ref{Tab.2} illustrate the best achievable test accuracy of our proposed method and compared group on different datasets (Sup. represents the supervised methods). From TABLE \ref{Tab.1} we can observe that, on Fashion-MNIST, UM-pFSSL outperforms most of the Non-IID FL methods and personalized methods with highly heterogeneous data. With the heterogeneity decreasing ($\alpha$ increasing), the baseline methods and Non-IID FL methods obtain noticeable performance improvement (about 2\% $\sim$ 7\%) while the personalized methods suffer the accuracy decline. Compared with other peer schemes, UM-pFSSL is able to retain a well-accepted accuracy $>$ 79\% with data dissimilarity across clients varying. TABLE \ref{Tab.2} depicts the result on CIFAR-10 dataset. In this task, personalized methods also show that performance degrades as the heterogeneity increases. The UM-pFSSL obtains comparable performance and even higher accuracy than baseline method FedAvg-SL, although the performance improvement is slightly weaker than the Non-IID methods FixAvg, FixProx and FixBN. The reason behind is that, under the highly heterogeneous setting, the helper selection mechanism can pick out the most related clients that can give most confident prediction about local unlabeled data. On the contrary, in nearly homogeneous context, almost all clients contain the similar distribution, the client with the largest labeled dataset will predominate the helper list of all clients. The potential knowledge aggregation is weakened in this scenario; thus, the performance does not show significant growth.
\begin{figure}[htb]
   \centering
   \subfigure[Fashion-MNIST]{\includegraphics[width=1.7in]{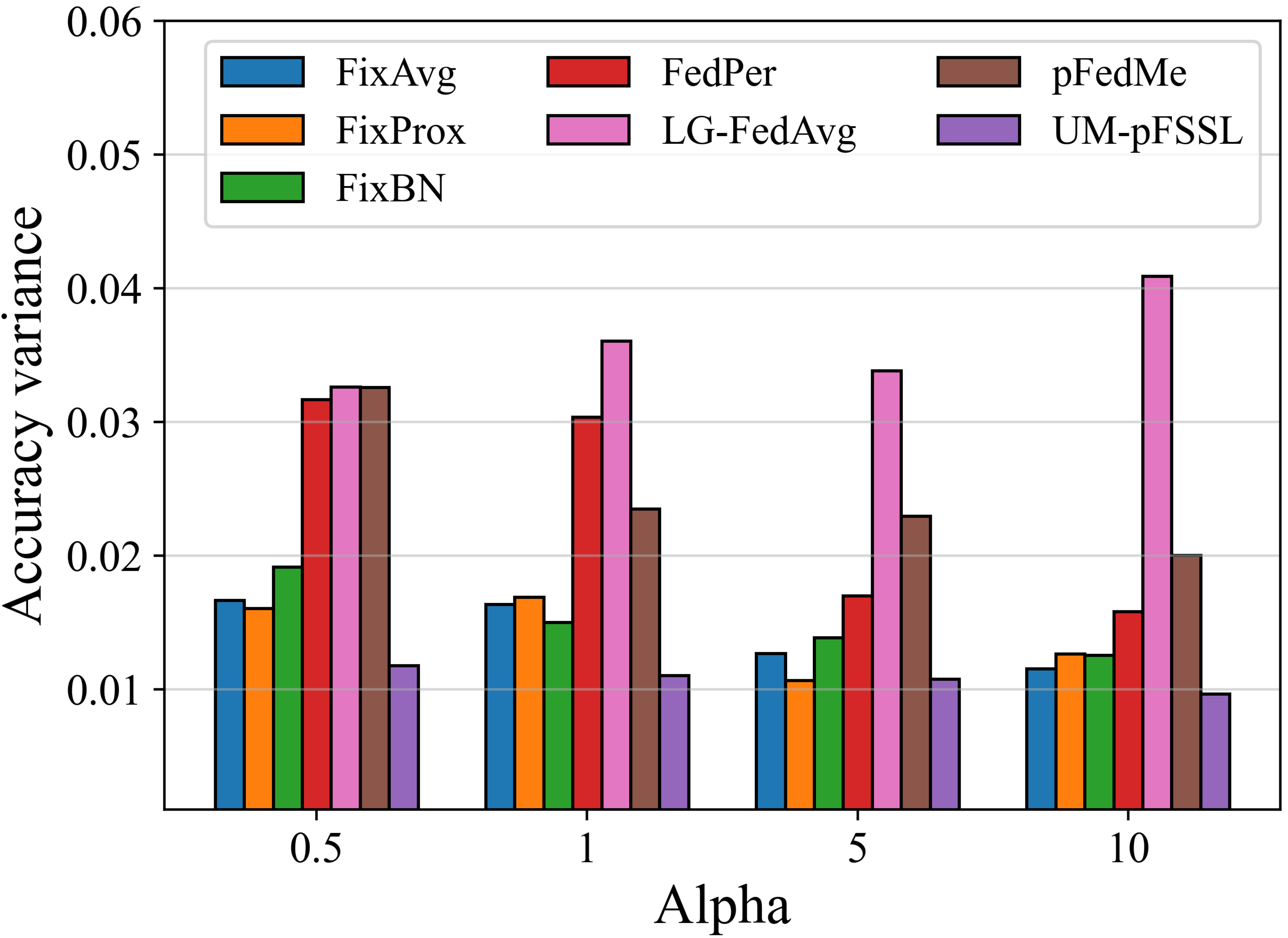}}
   \subfigure[CIFAR-10]{\includegraphics[width=1.7in]{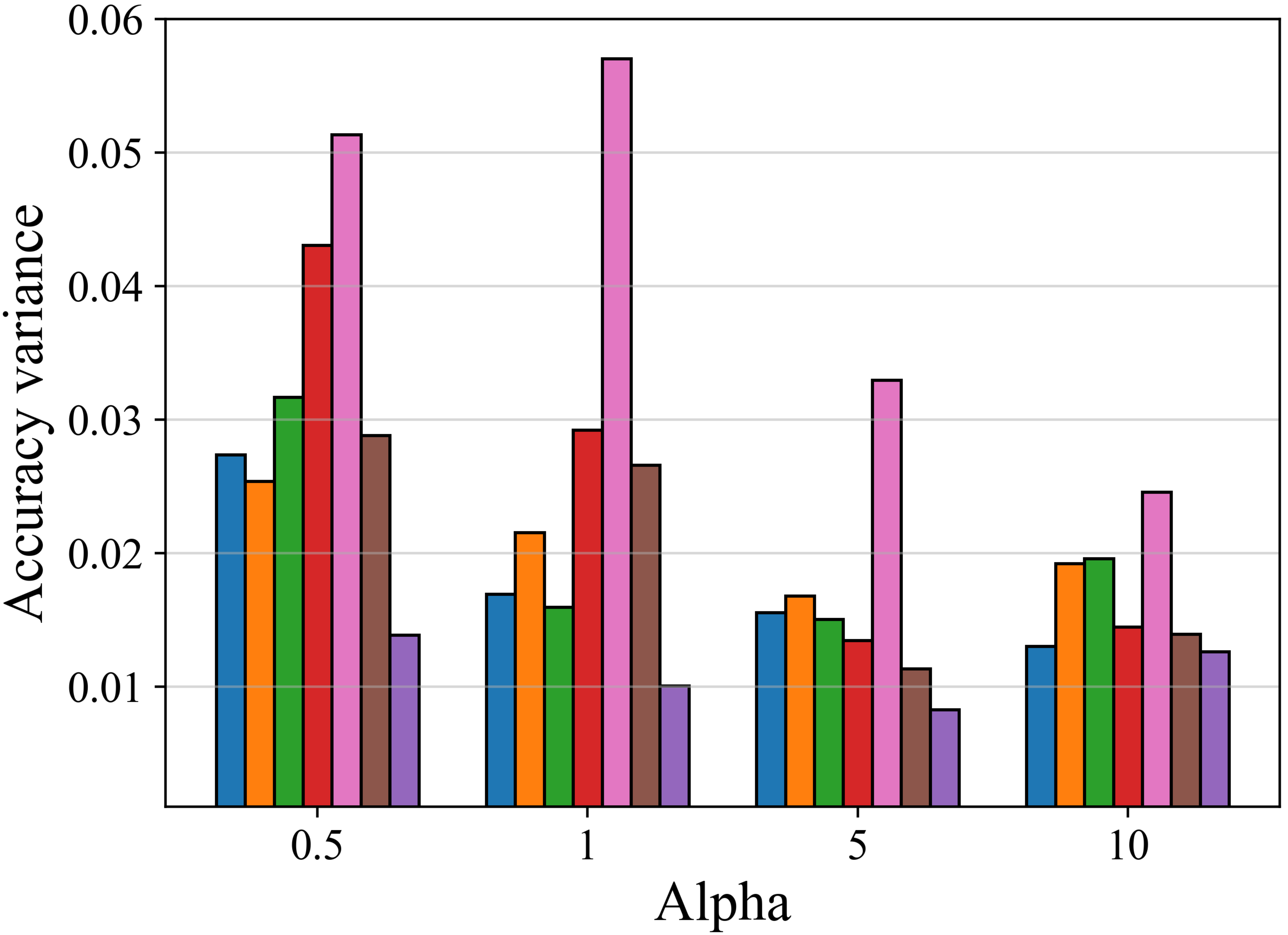}}
   \caption{Final test accuracy variance of FixAvg, FixProx, FixBN, FedPer, LG-FedAvg, pFedMe and UM-pFSSL on Fashion-MNIST and CIFAR-10.}
   \label{Fig.6}
\end{figure}
\par \textbf{Learning efficiency:} Figs. \ref{Fig.4} and \ref{Fig.5} visualize the curve of the average validation accuracy during the training procedure on Fashion-MNIST and CIFAR-10, respectively. As shown in Fig. \ref{Fig.4}, the personalized methods can acquire better accuracy improvement in the first few rounds, especially the LG-FedAvg and UM-pFSSL. As the training rounds increase, Non-IID methods show precipitous accuracy upgrading and quickly turn to flat convergence, while the LG-FedAvg and UM-pFSSL have a more smooth training curve. Additionally, compared to other personalized methods, UM-pFSSL keeps the performance superiority as the heterogeneity decreases. As shown in Fig. \ref{Fig.5}, for more complicate learning task CIFAR-10, the Non-IID methods tend to overfit after 100 rounds and end up to inferior validation accuracy. The UM-pFSSL shows more smooth learning curve and better performance than personalized methods and Non-IID methods with highly heterogeneous data distribution ($\alpha=0.5$). The comparisons indicate that our proposed methods can achieve more stable training and benefit from the heterogeneity of data. The main reason is the pseudo-label selection approach can collect the most useful information from related helpers. With higher the heterogeneity in distributed datasets, each selected helper can only provide confident predictions to a fraction of unlabeled data, thus the potential knowledge aggregation integrates more helpers and obtain more preferable and robust performance.

\par \textbf{Performance fairness:} In personalized FL system, the divergence of local performance between clients can be utilized to qualify the fairness of the FL methods. Specifically, we adopt the variance of the test accuracy of different clients to compare the fairness of our proposed UM-pFSSL with other methods. As shown in Fig. \ref{Fig.6}(a), the most biased method is the LG-FedAvg, which achieves the secondary performance among the personalized methods. In comparison, other methods obviously achieve more unbiased results with the $\alpha$ increases. From Fig. \ref{Fig.6}(b), most of the compared schemes obtain evident decrease in variance with $\alpha$ increases, and the UM-pFSSL reaches the lowest accuracy variance with $\alpha=5$. Based on these phenomena, we can easily obtain that, on both datasets, the UM-pFSSL achieves the least biased results with different degree of heterogeneity. Combined with the results in Figs. \ref{Fig.4} and \ref{Fig.5}, the UM-pFSSL shows both superior accuracy and comparable fairness in comparison with Non-IID and personalized methods, especially with massive data heterogeneity.

\par \textbf{Communication cost:} With a greedily searching strategy in helper-selection stage, the messages transmitted through network have the size:
\begin{equation}
\operatorname{Cost}_1=\tau K|w|n + \tau K(K-1)|w|n=\tau K^2|w|n,
\end{equation}
which includes the model upload of $\tau K$ selected clients and peer model download ($\tau$ is the sample rate for training nodes, $K$ is the total number of clients); $|w|$ denotes parameter size of the uniform neural architecture, and $n$ is the number of rounds.

\par With the helper selection protocol, in first $F$ rounds, we replace $R$ helpers in the helper list and update the models of total $M$ helpers every $\nu$ rounds through cloud server (these operations are executed on all $K$ clients). In some cases, the update or replacement should be skipped if there is no modification of existing $M$ helpers or suitable substitutes for $R$ helpers are not found. Therefore, the upper bound of the communication cost is
\begin{equation}
\operatorname{Cost}_{2}=K M|w| \frac{n}{\nu}+ F R K |w|+\tau K|w| n.
\end{equation}

\par The difference of $\operatorname{Cost}_{2}$ and $\operatorname{Cost}_{1}$ can be derived as
\begin{equation}
\Delta=\operatorname{Cost}_{1}-\operatorname{Cost}_{2}=\left[\tau K n - FR - \frac{M+\tau \nu}{\nu} n\right] K |w|.
\end{equation}

\par Physically, $FR$ in Eq. (29) represents the number of helpers that each user can search for during the entire training period. In the experiments, we set $F$ to 30 and $R$ to 2 by default. In such a case, we have $\tau K n = 2000$, $FR=60$ and $n(M+\tau \nu)/\nu=120$. There is a significant gap between $\operatorname{Cost}_{1}$ and $\operatorname{Cost}_{2}$, more than 90\% of the communication cost is saved. To intuitively reveal the effect of $FR$, we evaluate the communication cost of UM-pFSSL on two datasets by setting $F$ to different number with $\alpha=0.5$. As shown in Fig. \ref{Fig.7}, higher $F$ introduces more communication burden of the system and but always improves the test accuracy accordingly; while $F$ exceeds 30, the performance improvement effect gradually weakens. The results show that, our scheme can save considerable communication overhead with acceptable performance sacrifice.

\begin{figure}[htb]
   \centering
   \subfigure[Fashion-MNIST]{\includegraphics[width=1.7in]{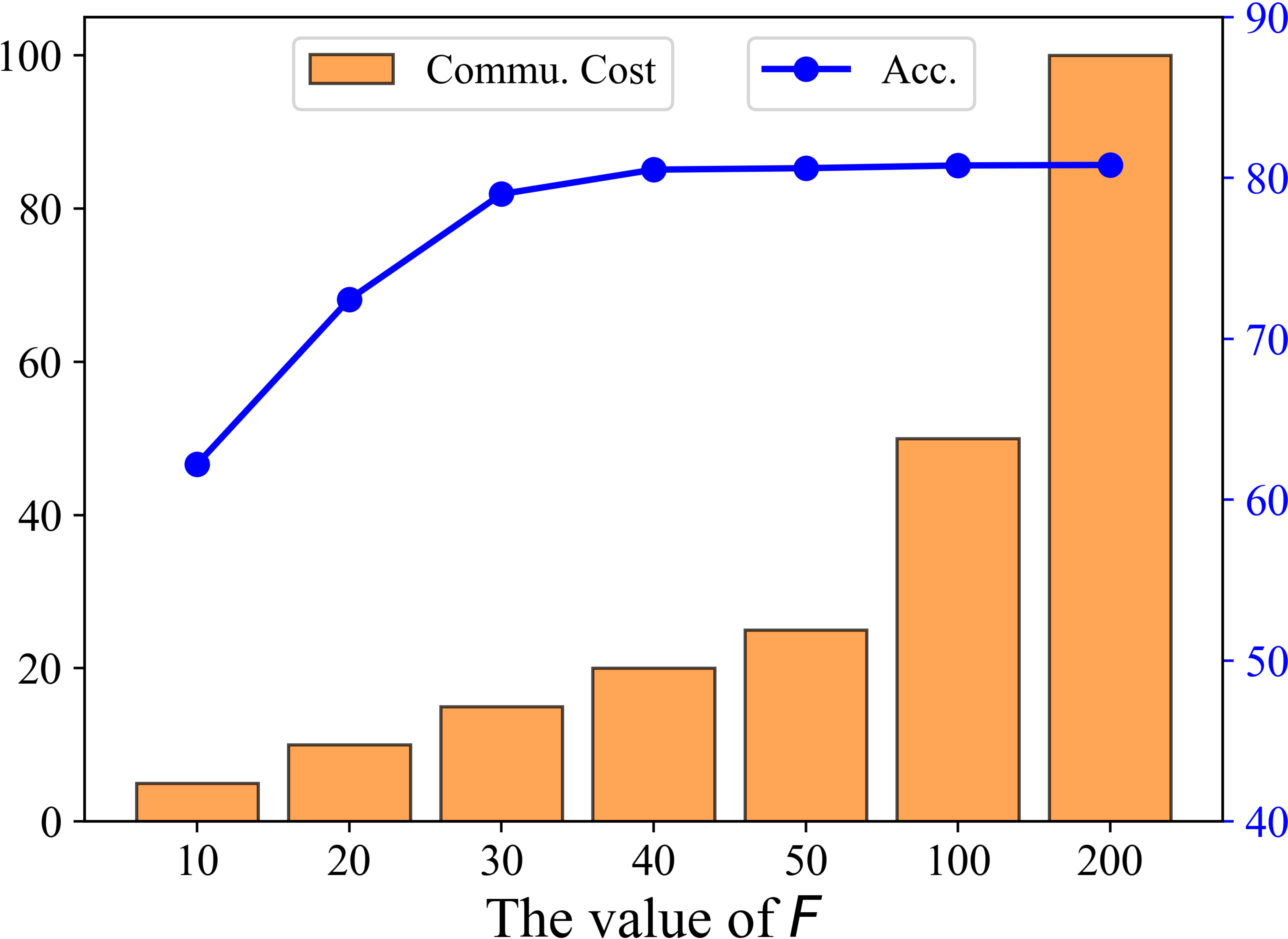}}
   \subfigure[CIFAR-10]{\includegraphics[width=1.7in]{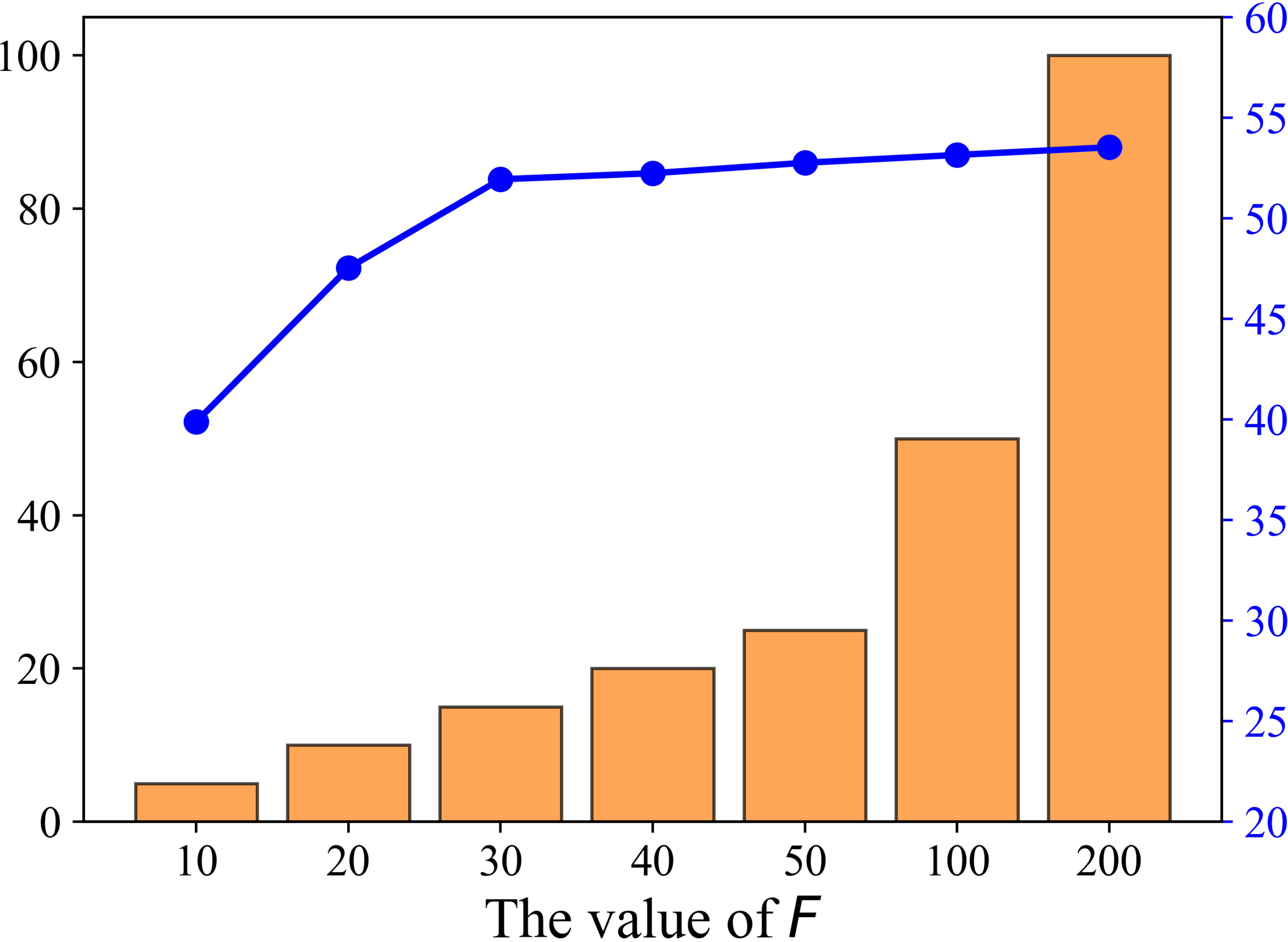}}
   \caption{The effect of $F$ on communication cost and model performance during helper selection. We take the communication cost of $F=200$ as the upper bound, and the communication costs represented in the figure are the percentage value of the actual cost (with different $F$) over this upper bound.}
   \label{Fig.7}
\end{figure}

\subsection{Ablation Study}
\par To show the detailed contributions of the components in UM-pFSSL, we conduct ablation experiments with two datasets. In detail, we separately test the effect on the final model performance using the two components in Eq. (13), i.e., the entropy term (EN) and the test accuracy term (TA). As shown in TABLE \ref{Tab.3}, the entropy term outperforms the accuracy term in all sorts of situations, and the fusion of them can even further improve the performance.

\begin{table}[htbp]
   \centering
   \caption{The results of key components in UM-pFSSL.}
   \setlength{\tabcolsep}{3.5mm}{\begin{tabular}{ccccc}
      \toprule
       ~ & \textbf{Dataset}      & TA  & EN & EN + TA\\
      \midrule
      \multirow{2}*{$\alpha=0.5$}& Fashion-MNIST & 71.39 & 73.25 & 79.00  \\
      					~ & 		   CIFAR-10  & 41.97 & 45.82 & 51.14  \\
	  \midrule
	  \multirow{2}*{$\alpha=1$}& Fashion-MNIST  & 75.11 & 77.74 & 80.93  \\
      					~ & 		   CIFAR-10 & 46.16 & 48.93 & 52.24  \\
	  \midrule
	  \multirow{2}*{$\alpha=5$}& Fashion-MNIST  & 76.39 & 78.25 & 81.16  \\
      					~ & 		   CIFAR-10 & 47.63 & 49.32 & 52.83  \\
	  \midrule
	  \multirow{2}*{$\alpha=10$}& Fashion-MNIST  & 76.83 & 79.19 & 81.49  \\
      					~ & 		   CIFAR-10  & 47.95 & 50.47 & 54.03  \\
      \bottomrule
   \end{tabular}}
   \label{Tab.3}
\end{table}

The pseudo-label error rate of the key compoments with $\alpha=0.5$ is visualized in Fig. \ref{Fig.8}. As the training rounds increase, UM-pFSSL with EN shows more smooth error decreasing than the one with TA. While the UM-pFSSL with TA presents a rapid dropping in 50-100 rounds and then oscillates around a stable value. The combination of the two shows the same downward trend as EN and has a lower error rate than both. This result clearly tells us that with our designed corresponding metric, the selected helpers are able to annotate unlabeled data effectively.

\begin{figure}[htb]
   \centering
   \subfigure[Fashion-MNIST]{\includegraphics[width=1.7in]{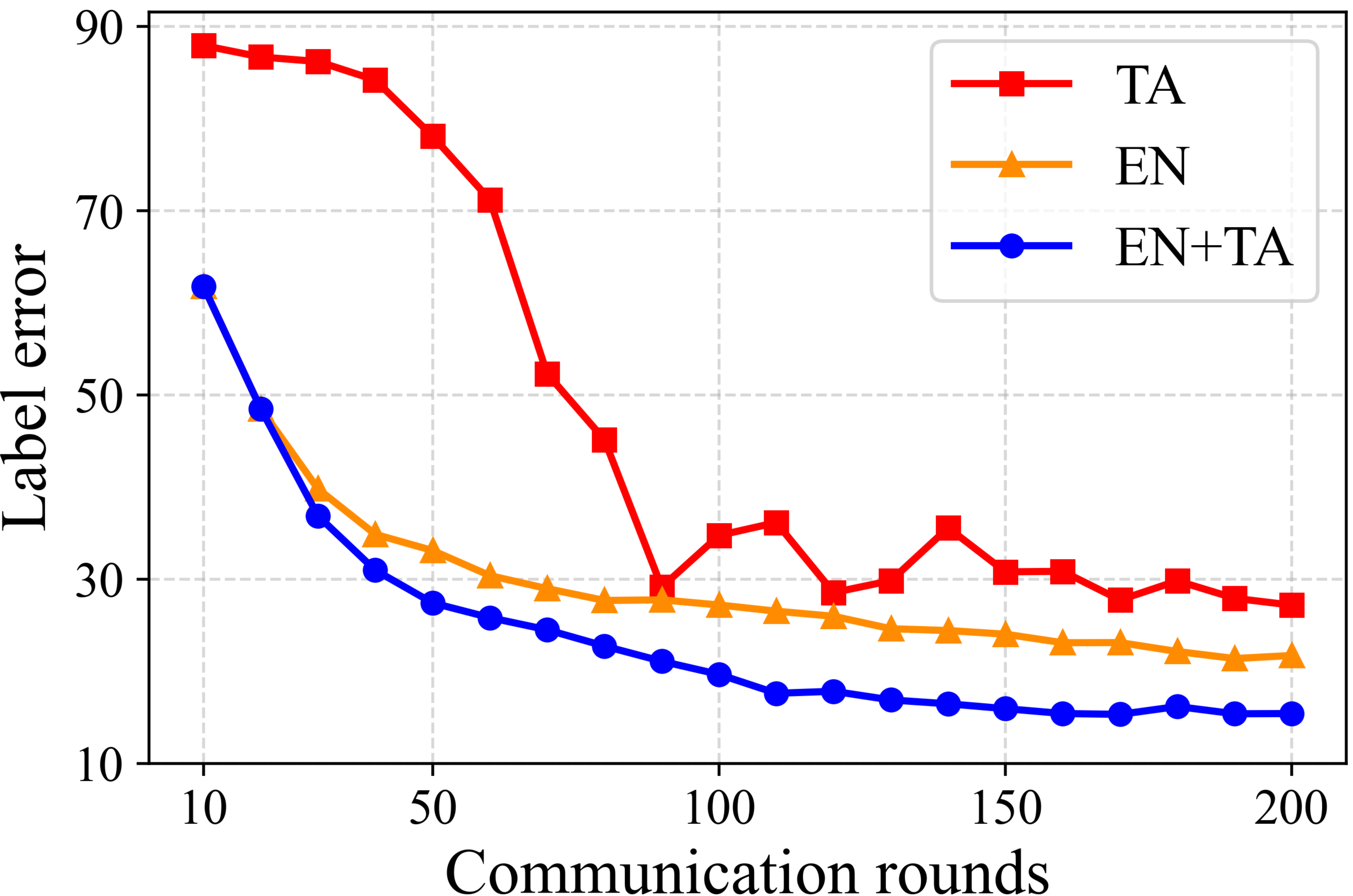}}
   \subfigure[CIFAR-10]{\includegraphics[width=1.7in]{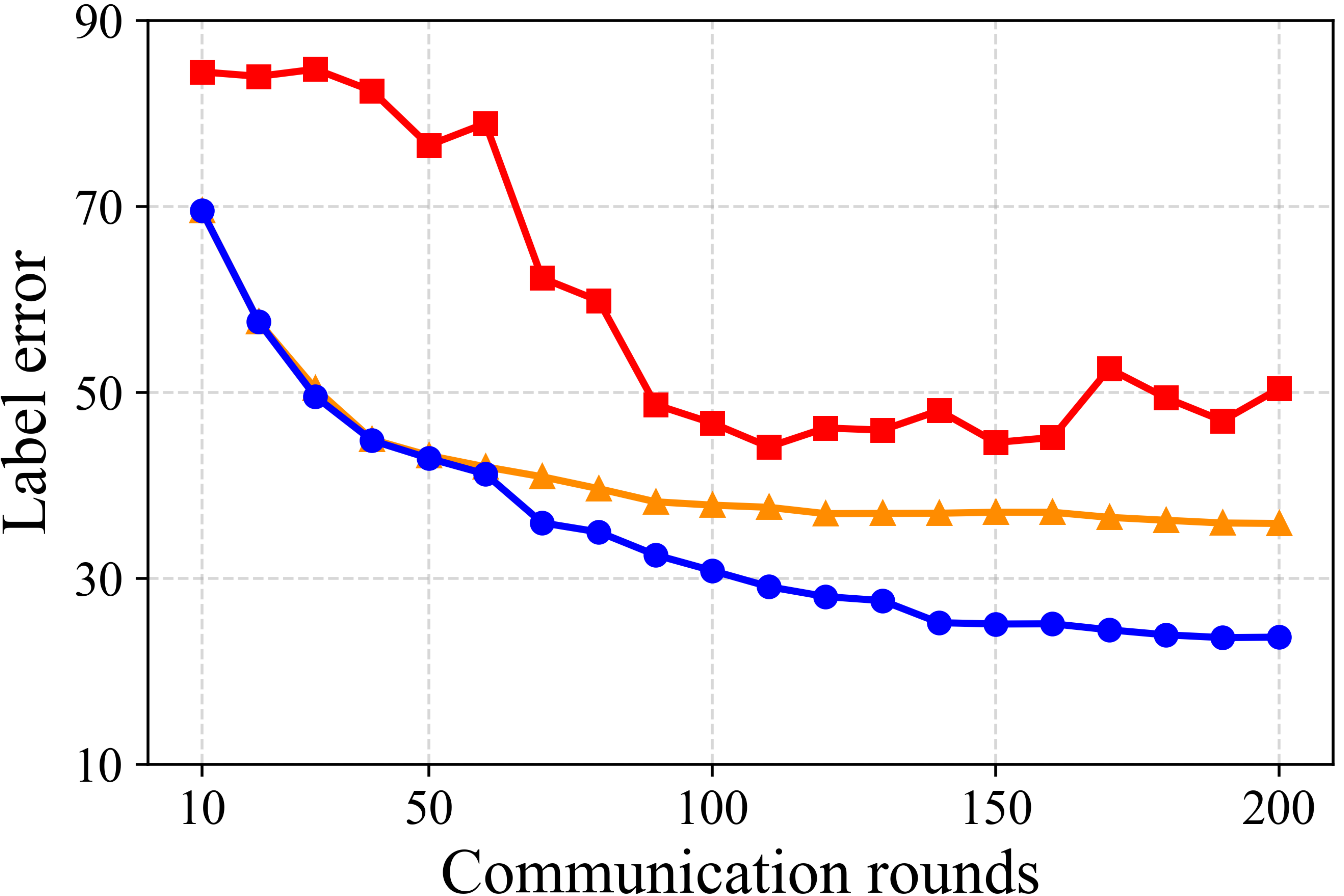}}
   \caption{Error rate for labels provided by helpers when using key components.}
   \label{Fig.8}
\end{figure}

\par We also give the distribution of data features with less heterogeneity ($\alpha=5$). In this situation, we take the output of the final hidden unit as the latent feature of input data. Since the models are fully personalized, we pick the optimal model in each case for visualization. As depicted in Fig. \ref{Fig.9}, with the combination of EN and AT, the model manifestly improves the discrimination on hard samples (the green and red classes in the figure).

\begin{figure}[htb]
   \centering
   \includegraphics[width=3in]{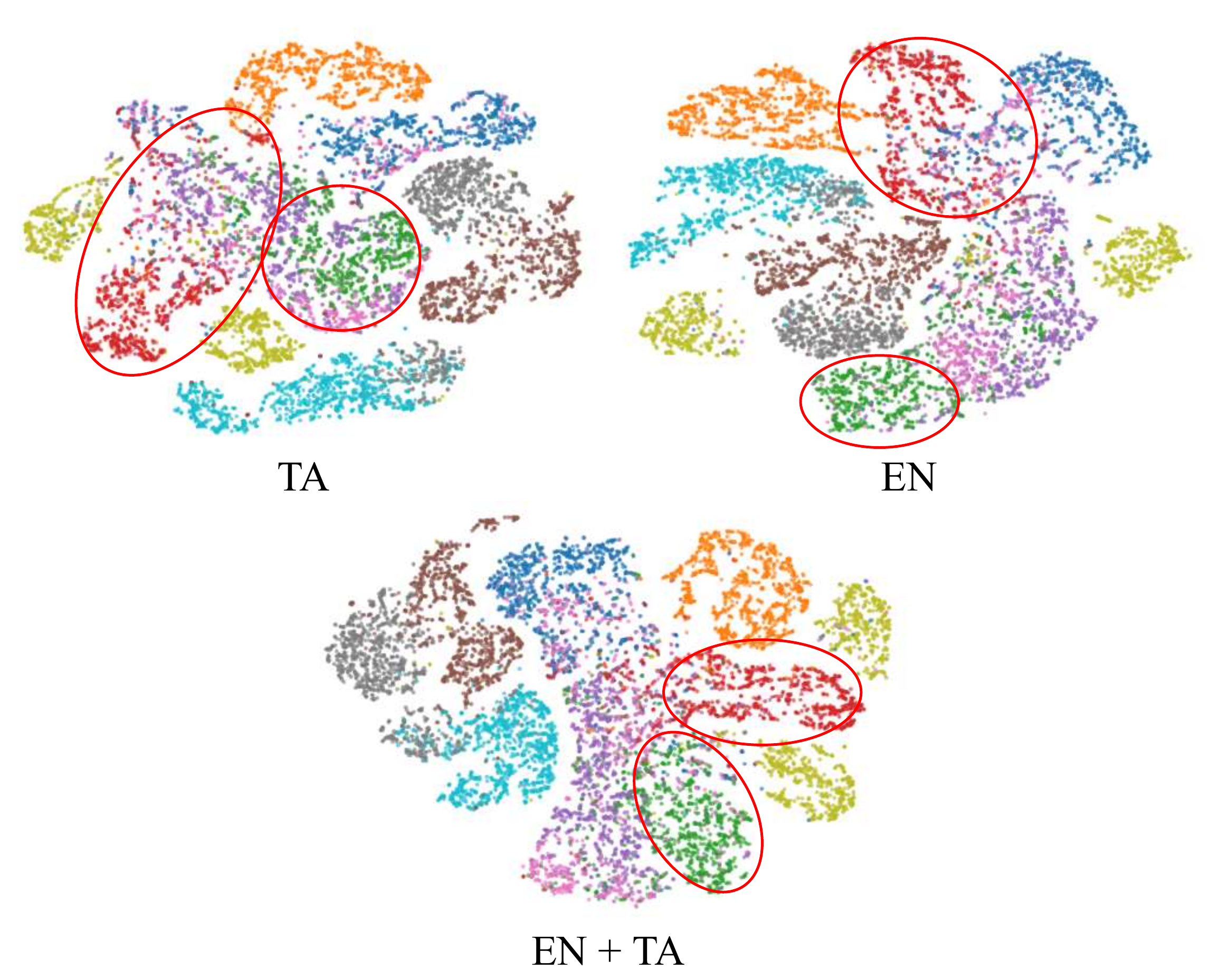}
   \caption{t-SNE plots: the distribution of latent features encoded by trained models with $\alpha=5$ on Fashion-MNIST.}
   \label{Fig.9}
\end{figure}

\subsection{Hyper-parameter Analysis}
\par Firstly, we discuss the influence of the $M$ on the final experimental results in helper selection. We adjust this parameter from 3 to 15, and visualize the results in Fig. \ref{Fig.10}. The results show that increasing the number of helpers has obvious benefits for systems with less heterogeneity. For those with higher degree of heterogeneity, increasing helpers may lead to a decrease in model performance. This stems from the model aggregation procedure described in Section  \ref{sec:paradigm}. With this operation, aggregating models from heterogeneous clients reduces the convergence speed and training stability of the model.

\begin{figure}[htb]
   \centering
   \subfigure[Fashion-MNIST]{\includegraphics[width=1.7in]{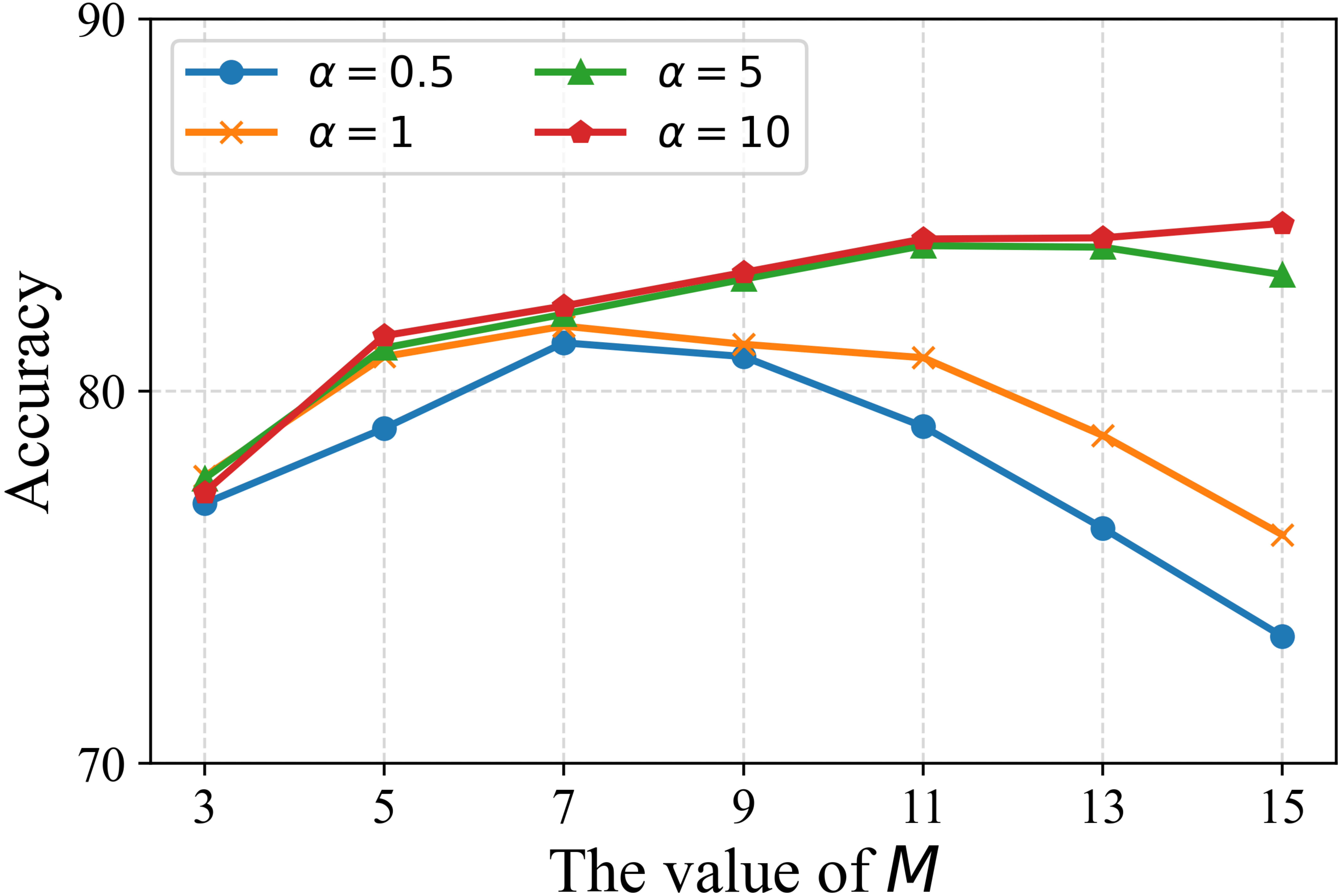}}
   \subfigure[CIFAR-10]{\includegraphics[width=1.7in]{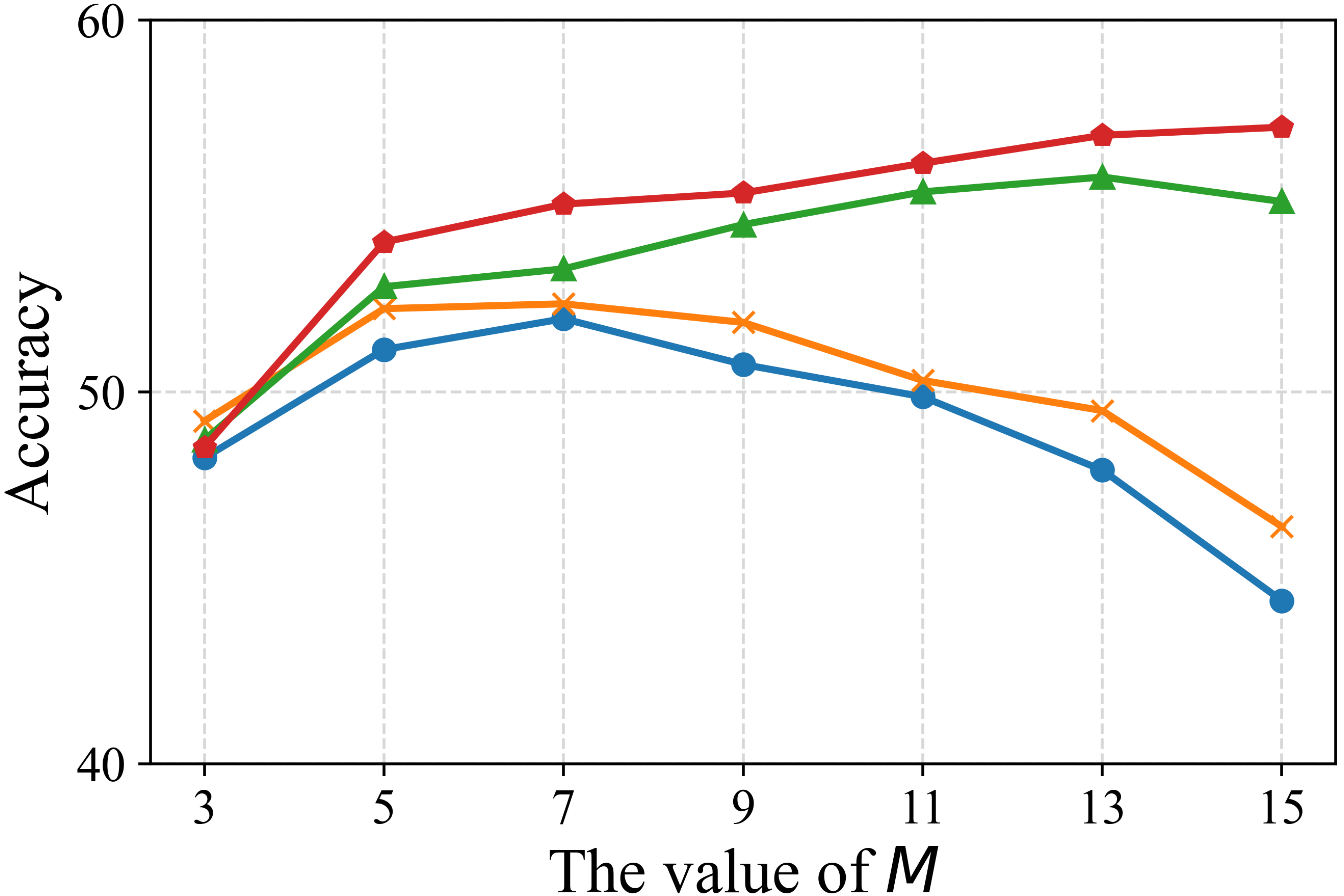}}
   \caption{The influence of the $M$ on the final experimental results.}
   \label{Fig.10}
\end{figure}

\par The frequency of updating helpers' model $\nu$ also affects performance and communication efficiency. We adjust $\nu$ from 1 to 20, and evaluate communication cost and model performance on Fashion-MNIST with $\alpha=0.5$. The results are shown in TABLE \ref{Tab.4}. A larger $\nu$ can save certain communication costs, but as the value of $\nu$ gets larger, the degree of performance degradation also increases.

\begin{table}[htbp]
   \centering
   \caption{INFLUENCE OF $\nu$ ON UM-pFSSL.}
   \setlength{\tabcolsep}{2mm}{\begin{tabular}{c|ccccc}
      \toprule
      \textbf{$\nu$} & 1 (baseline) & 5 & 10 (default) & 15 & 20\\
      \midrule
      Accuracy & 80.32 & 79.85 & 79.00 & 76.68 & 73.21\\
      Cost & 100\% & 26\% & 16.7\% & 13.5\% & 12\% \\
      \bottomrule
   \end{tabular}}
\label{Tab.4}
\end{table}

\par According to Eq. (29), the client sampling rate $\tau$ linearly affects the communication load of the system. Same as TABLE \ref{Tab.4}, we adjust it from 0.05 to 1, and show the results in TABLE \ref{Tab.5}. Different from $\nu$, a larger $\tau$ means more local training performed by each client. The performance improvement brought by increasing $\tau$ is more significant than that of $\nu$.

\begin{table}[htbp]
   \centering
   \caption{INFLUENCE OF $\tau$ ON UM-pFSSL.}
   \setlength{\tabcolsep}{2mm}{\begin{tabular}{c|ccccc}
      \toprule
      \textbf{$\tau$} & 0.05 & 0.1 (default) & 0.2 & 0.5 & 1.0 (baseline) \\
      \midrule
      Accuracy & 77.46 & 79.00 & 81.75 & 82.03 & 83.21\\
      Cost & 5.7\% & 10.7\% & 20.6\% & 50.4\% & 100\% \\
      \bottomrule
   \end{tabular}}
\label{Tab.5}
\end{table}

\subsection{Experiments with Real-world Dataset}
\par In order to examine the performance of our method in real-world applications, we additionally adopt two medical imaging datasets from MedMNIST collection [42]: OrganMNIST(Axial) and PathMNIST datasets. OrganMNIST consists of 11 types of body organs for classification task. PathMNIST is comprised of 9 types of tissues for classification task. We set the number of engaged clients $K$ to 20 and client sample rate $\tau$ to 1. The other settings are the same as Subsection A. TABLE \ref{Tab.7} and TABLE \ref{Tab.8} give the results of the test accuracy on the two medical datasets. Clearly, our method exhibits superior performance than compared methods, with margins $>$ 2\% at varying degrees of heterogeneity. In addition, we aggregate the test results on all clients and compare with the mixture of all Non-FL methods, which is the ensemble of FixAvg, FixProx and FixBN by averaging their output. The detailed classification results through the confusion matrices are shown in Fig. \ref{Fig.11} (with $\alpha=0.5$). From this figure, we can learn that, UM-pFSSL achieves higher worst-case accuracy and superior accuracy on more than half of the classes than mixed method. These additional experiments demonstrate that our scheme can achieve superior performance on real-world datasets compared to related methods.

\begin{table}[htbp]
   \centering
   \caption{THE COMPARISION OF BEST TEST ACCURACY ON PathMNIST WITH DIFFERENT $\alpha$.}
   \setlength{\tabcolsep}{4.5mm}{\begin{tabular}{ccccc}
      \toprule
      \textbf{Algorithms} & $\alpha$=0.5  & $\alpha$=1  & $\alpha$=5  & $\alpha$=10\\
      \midrule
	  FixAvg & 70.48 & 72.16 & 75.63 & 77.28 \\
      FixProx & 71.36 & 73.57 & 76.07 & 77.85 \\
      FixBN  & 72.46 & 75.19 & 78.38 & 79.33 \\
     \midrule
      FedPer & 68.09 & 64.71 & 62.36 & 60.13     \\
	 LG-FedAvg & 74.25	& 73.49 & 73.41 & 72.71 \\
      pFedMe & 64.19 & 63.73 & 62.32 & 60.17\\
     \midrule
     UM-pFSSL & \textbf{78.64} & \textbf{79.83} & \textbf{80.94} & \textbf{81.47}\\
      \bottomrule
   \end{tabular}}
   \label{Tab.7}
\end{table}

\begin{table}[htbp]
   \centering
   \caption{THE COMPARISION OF BEST TEST ACCURACY ON OrganMNIST WITH DIFFERENT $\alpha$.}
   \setlength{\tabcolsep}{4.5mm}{\begin{tabular}{ccccc}
      \toprule
      \textbf{Algorithms}              & $\alpha$=0.5  & $\alpha$=1  & $\alpha$=5  & $\alpha$=10\\
      \midrule
	  FixAvg & 78.79 & 80.82 & 82.38 & 82.74 \\
      FixProx & 79.37 & 82.81 & 84.61 & 84.65    \\
      FixBN  & 81.54 & 83.28 & 85.52 & 85.93\\
     \midrule
      FedPer & 72.86 & 69.93 & 68.76 & 68.53   \\
	 LG-FedAvg & 75.29 & 74.78 & 73.62 & 72.97 \\
      pFedMe & 66.59 & 64.27 & 63.79 & 62.36\\
     \midrule
     UM-pFSSL & \textbf{86.62} & \textbf{87.23} & \textbf{87.68} & \textbf{88.53}\\
      \bottomrule
   \end{tabular}}
   \label{Tab.8}
\end{table}

\begin{figure}[htb]
   \centering
   \subfigure[PathMNIST + UM-pFSSL]{\includegraphics[width=1.7in]{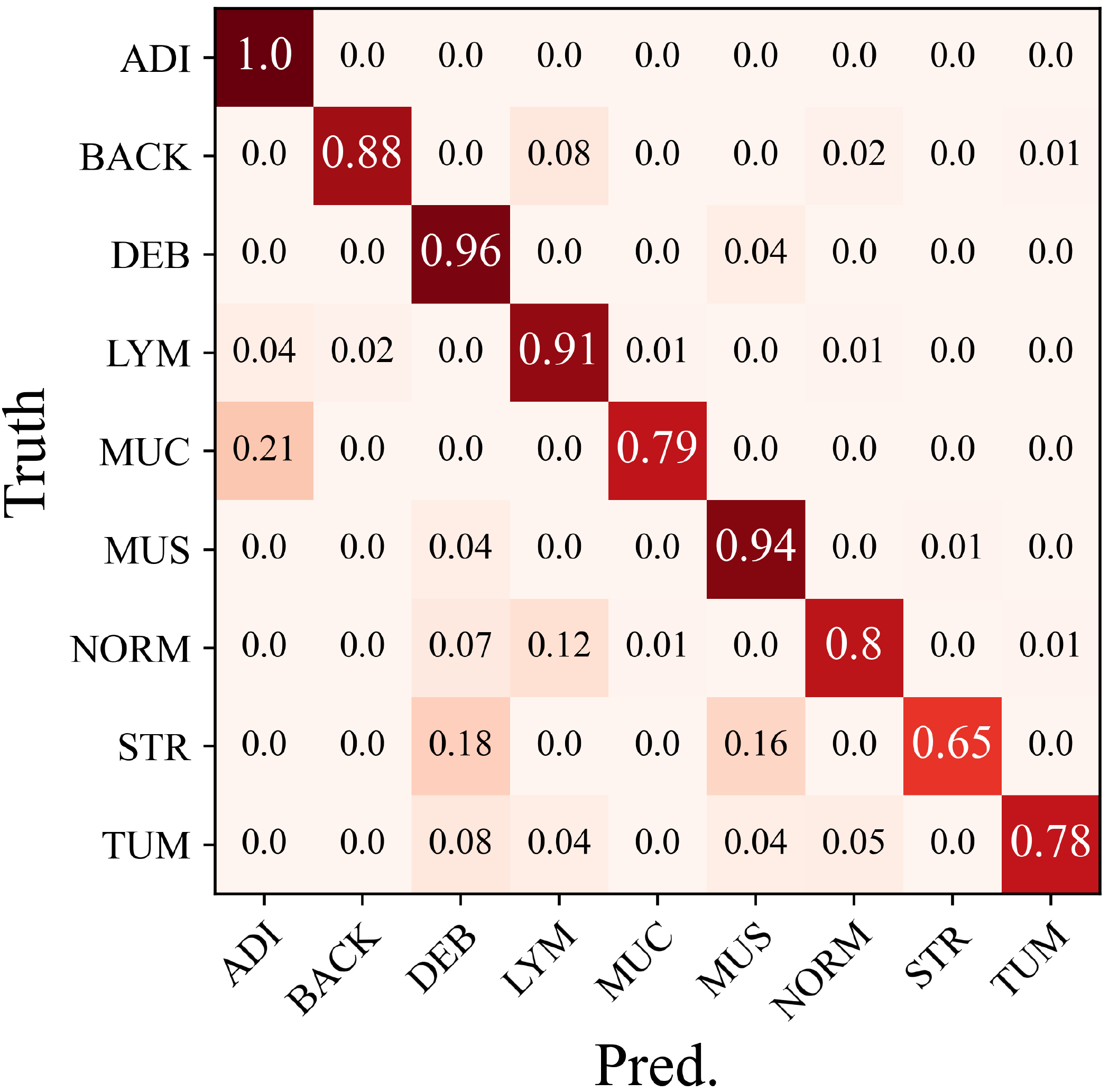}}
   \subfigure[PathMNIST + Mix]{\includegraphics[width=1.7in]{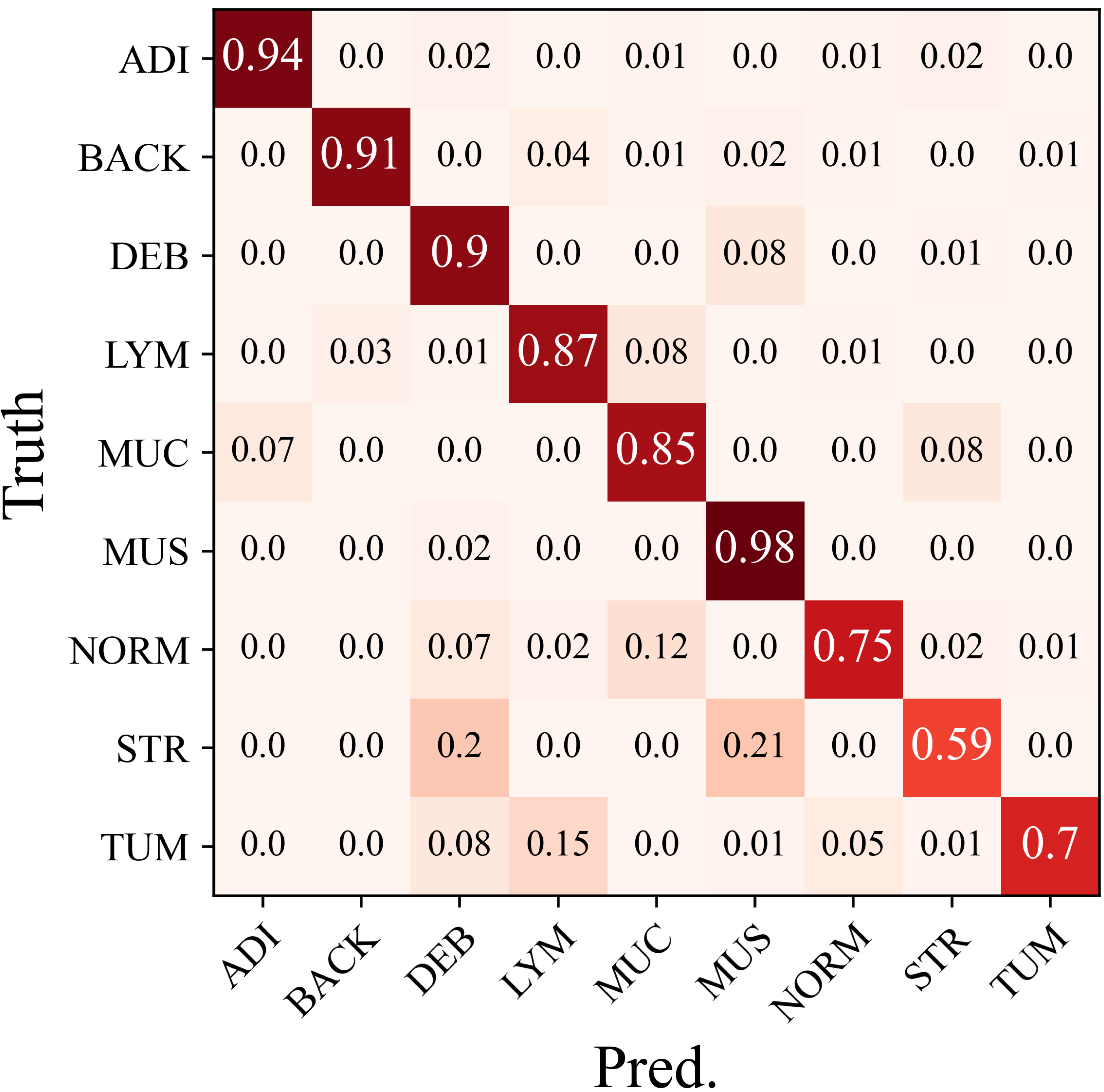}}
   \subfigure[OrganMNIST + UM-pFSSL]{\includegraphics[width=1.7in]{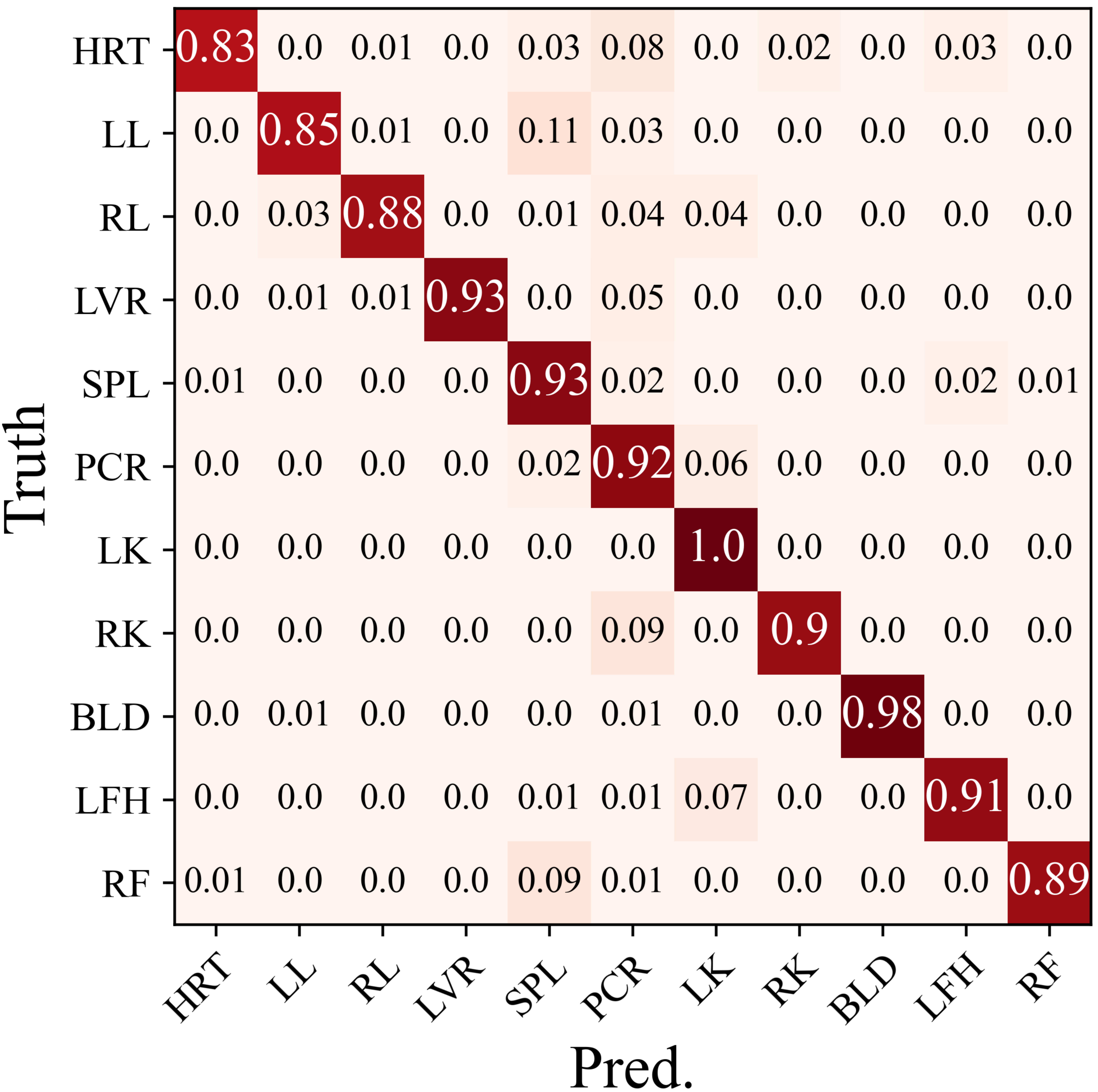}}
   \subfigure[OrganMNIST + Mix]{\includegraphics[width=1.7in]{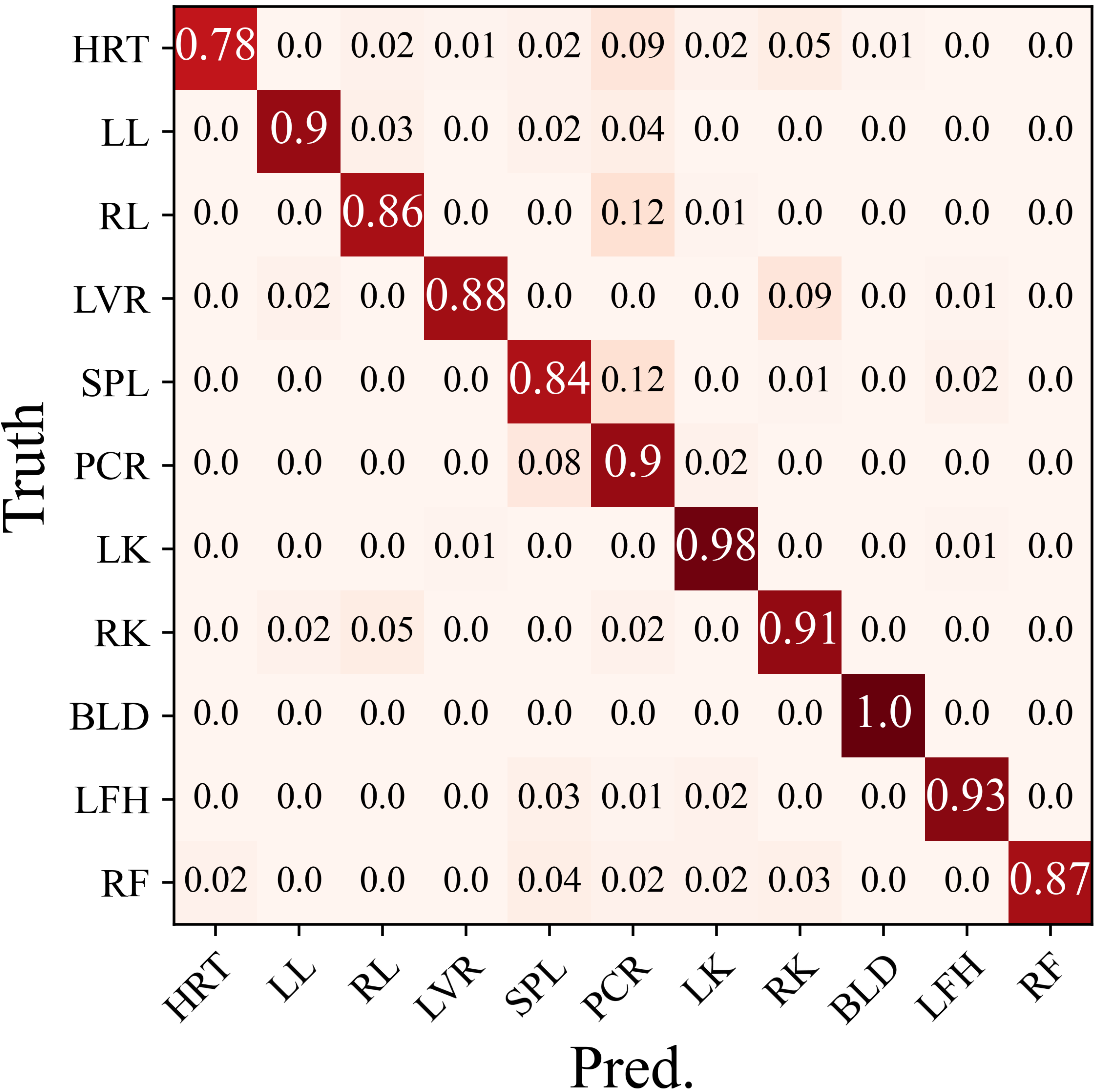}}
   \caption{The confusion matrices of UM-pFSSL and Mix method on different datasets. The vertical axis of the matrices represents the true label of the samples and the horizontal axis represents the predicted label of samples, each decimal is the proportion of the predicted label over true label. The labels on the axes are the abbreviations of the category names: ADI (adipose tissue), BACK (background), DEB (debris), LYM (lymphocytes), MUC (mucus), MUS (smooth muscle), NORM (normal colon mucosa), STR (cancer-associated stroma), TUM (colorectal adenocarcinoma epithelium). HRT (heart), LL (left lung), RL (right lung), LVR (liver), SPL (spleen), PCR (pancreas), LK (left kidney), RK (right kidney), BLD (bladder), LFH (left femoral head), RF (right femoral head).}
   \label{Fig.11}
\end{figure}
\section{Conclusions}\label{sec:conclusion}
\par In this work, we propose UM-pFSSL to address the issues caused by Non-IID data and label scarcity in pFL setting. In detail, a semi-supervised learning paradigm is developed, which allows each client to gather the knowledge from related neighbors to gain confident pseudo labels for local data. Based on this paradigm, we design a new client-relation measurement, which theoretically guarantees the generalization ability improvement of local model. To mitigate the heavy communication cost introduced by model exchange, we further present a helper selection strategy. Extensive experiments on two different datasets confirm that, compared to related works, our proposed method obtains the competitive achievements on performance and convergence rate than other related works, and benefits from data heterogeneity. In the future, we plan to extend this work into unsupervised learning based Non-IID FL scenario.

\section*{Acknowledgment}
This work was partially supported by the National Natural Science Foundation of China (61971235), the China Postdoctoral Science Foundation (2018M630590), the Jiangsu Planned Projects for Postdoctoral Research Funds (2021K501C), the 333 High-level Talents Training Project of Jiangsu Province, the 1311 Talents Plan of NJUPT, and the Postgraduate Research and Innovation Project of Jiangsu Province (KYCX21\_0803).

\vspace{-10 mm}
\begin{IEEEbiography}[{\includegraphics[width=1in,height=1.25in,clip,keepaspectratio]{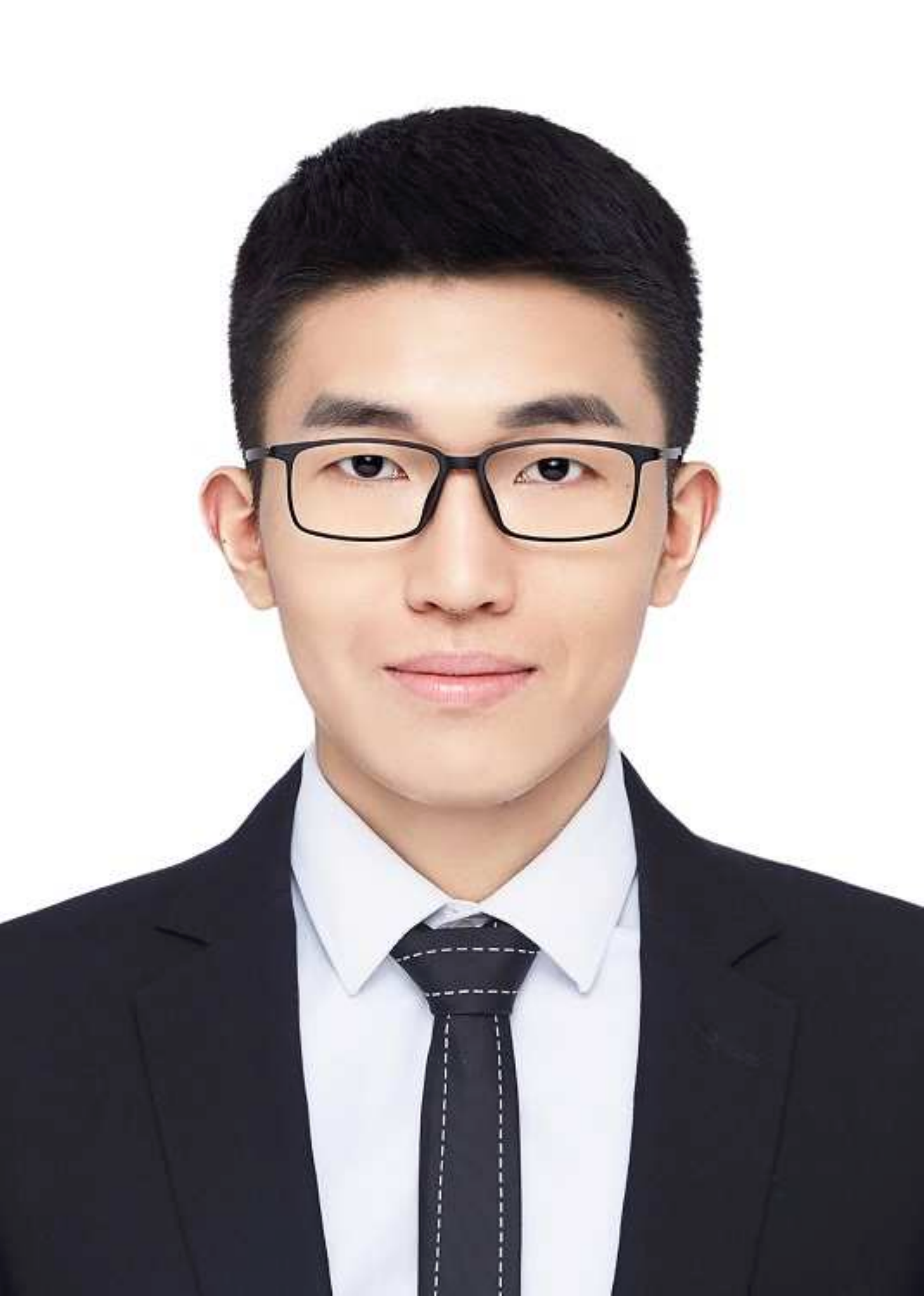}}]{Yanhang Shi}
received his B. E. degree in Internet of Things Engineering from Nanjing University of Posts and Telecommunications in 2019. He is currently pursuing a Ph.D. at the Nanjing University of Posts and Telecommunications. His research interests include deep learning and edge intelligence.
\end{IEEEbiography}
\vspace{-10 mm}
\begin{IEEEbiography}[{\includegraphics[width=1in,height=1.25in,clip,keepaspectratio]{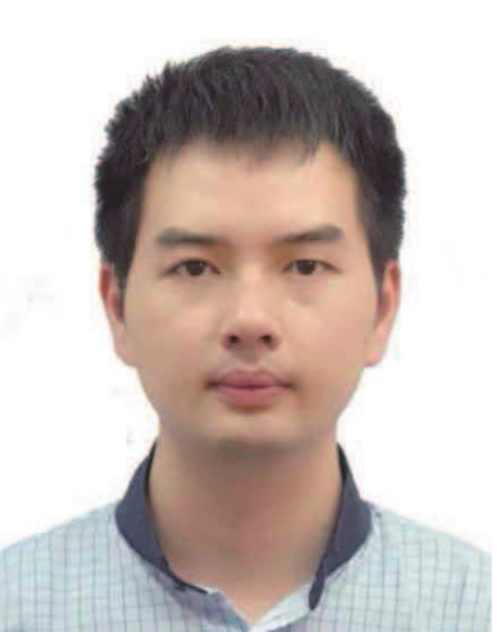}}]{Siguang Chen}
(M'17) is currently a Full Professor at Nanjing University of Posts and Telecommunications. He received his Ph.D. in Information Security from Nanjing University of Posts and Telecommunications, Nanjing, China, in 2011. He finished his Postdoctoral research work in City University of Hong Kong in 2012. From 2014 to 2015, he also was a Postdoctoral Fellow in the University of British Columbia. He has published more than 100 papers and applied 30 patents, serves as Editor of EAI Endorsed Transactions on Cloud Systems and Journal on Internet of Things, Guest Editor of Wireless Communications and Mobile Computing, Security and Communication Networks, and International Journal of Computer Networks and Communications, Corresponding Experts of Engineering Journal, and serves as General Co-Chair of ICAIS/ICCCS 2019-2022 Workshop and the Energy-Secure AIoT 2022. He also serves/served as Session Chair and TPC member of several international conferences, such as IEEE ICC, IEEE GLOBECOM, IEEE EDGE, and IEEE Cloud, etc. His current research interests are in the area of edge intelligence and AIoT.
\end{IEEEbiography}
\vspace{-10 mm}
\begin{IEEEbiography}[{\includegraphics[width=1in,height=1.25in,clip,keepaspectratio]{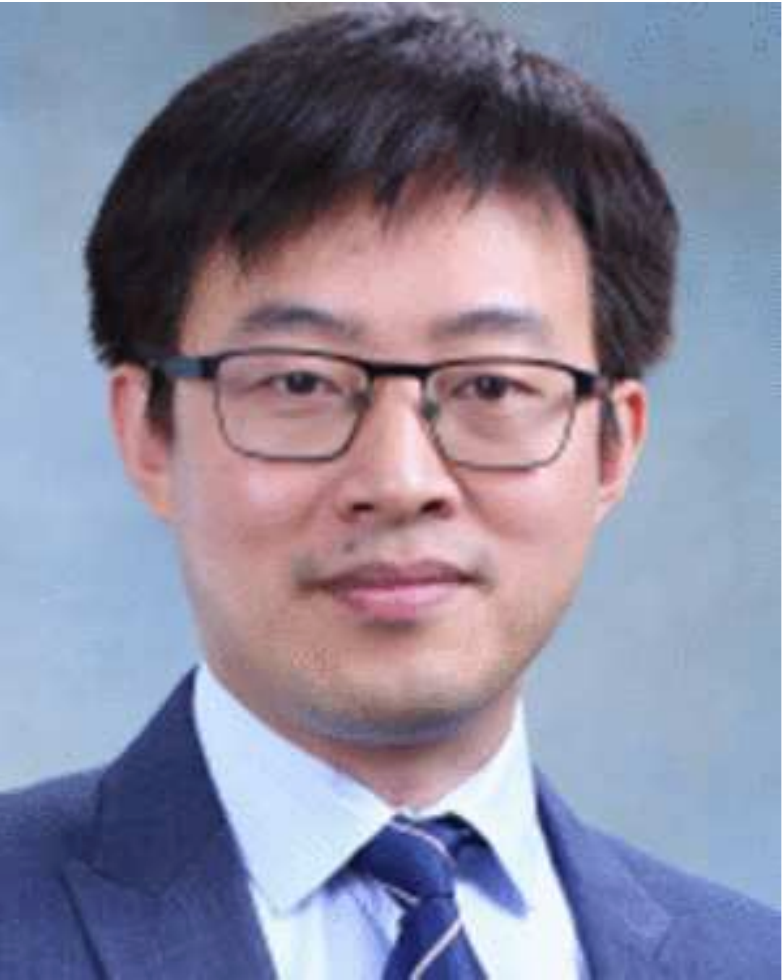}}]{Haijun Zhang}
 (M'13, SM'17) is currently a Full Professor and the Associate Dean of the School of Computer and Communication Engineering and Institute of Artificial Intelligence, University of Science and Technology Beijing, Beijing, China. He was a Postdoctoral Research Fellow with the Department of Electrical and Computer Engineering, University of British Columbia, Vancouver, BC, Canada. He was the Track Co-Chair of WCNC 2020/2021, Symposium Chair of GLOBECOM'19, TPC Co-Chair of INFOCOM 2018 Workshop on Integrating Edge Computing, Caching, and Offloading in Next Generation Networks, and General Co-Chair of GameNets'16. He is the Editor of the IEEE Transactions on Communications, IEEE Transactions on Network Science and Engineering, and IEEE Transactions on Green Communications and Networking. He was the recipient of the IEEE CSIM Technical Committee Best Journal Paper Award in 2018, IEEE ComSoc Young Author Best Paper Award in 2017, and IEEE ComSoc Asia-Pacific Best Young Researcher Award in 2019.
\end{IEEEbiography}

\end{document}